\documentclass{article}

\usepackage[preprint]{neurips_2026}
\PassOptionsToPackage{numbers, compress}{natbib}
\usepackage[utf8]{inputenc} 
\usepackage[T1]{fontenc}    
\usepackage{hyperref}       
\usepackage{url}            
\usepackage{booktabs}       
\usepackage{amsfonts}       
\usepackage{nicefrac}       
\usepackage{microtype}      
\usepackage{xcolor}         
\usepackage{graphicx}
\usepackage{subfigure}

\definecolor{purple}{RGB}{186,85,211}
\definecolor{orange}{RGB}{255,140,0}
\definecolor{blue}{RGB}{30,144,255}

\usepackage{booktabs}  
\usepackage{multirow}  
\usepackage{array}     
\usepackage{booktabs}
\usepackage{colortbl}
\usepackage[most]{tcolorbox}
\usepackage{xcolor}
\usepackage{amsmath}
\usepackage{amssymb}
\usepackage{mathtools}
\usepackage{amsthm}
\usepackage{thm-restate}

\usepackage[capitalize,noabbrev]{cleveref}
\usepackage{graphicx}
\usepackage{algorithm}
\usepackage{algorithmic}
\usepackage{amssymb}
\usepackage{wrapfig}
\usepackage{hyperref}
\usepackage{cleveref}
\usepackage{dsfont}
\usepackage{breqn}
\usepackage{wrapfig}

\crefname{table}{table}{tables}
\Crefname{table}{Table}{Tables}
\crefname{figure}{figure}{figures}
\Crefname{figure}{Figure}{Figures}
\Crefname{algo}{Algorithm}{Algorithms}

\theoremstyle{plain}

\theoremstyle{definition}

\theoremstyle{remark}

\title{Learning over Positive and Negative Edges with \\Contrastive Message Passing}

\author{%
  Peter Pao-Huang \\
  Department of Computer Science\\
  Stanford University\\
  \texttt{peterph@cs.stanford.edu} \\
  \And
  Charilaos I. Kanatsoulis \\
  Department of Computer Science\\
  Stanford University \\
  \texttt{charilaos@cs.stanford.edu} \\
  \AND
  Michael Bereket \\
  Department of Computer Science\\
  Stanford University \\
  \texttt{mbereket@cs.stanford.edu} \\
  \And
  Jure Leskovec \\
  Department of Computer Science\\
  Stanford University \\
  \texttt{jure@cs.stanford.edu} \\
}

\begin{document}

\maketitle

\begin{abstract}
  Conventional approaches to learning on graphs involve message passing along existing (i.e., \textit{positive}) edges to update node features. However, these approaches often disregard the potentially valuable information contained in the absence (i.e., \textit{negative}) of edges. Here, we theoretically analyze the value of negative edges in graph representations and prove that in settings of low label rates, high homophily, and high edge density, access to negative edges provides significant information gain over using only positive edges. Motivated by this insight, we introduce Contrastive Message Passing (CMP), a general message passing architecture that enable graph neural network layers to reason over positive and negative edges. By imposing soft positive semidefinite constraints on the learnable weights, our approach differentially applies similarity-preserving transformations to positively connected nodes and dissimilarity-inducing transformations to negatively connected nodes. Over simulated and real datasets in varying data regimes, CMP consistently outperforms baselines in low-label settings when negative edges are informative. 
\end{abstract}

\section{Introduction}
Modern graph neural networks (GNNs) commonly learn from graph-structured data by iteratively passing messages between nodes (aka \textit{Message Passing Graph Neural Networks}) to learn representations used for downstream tasks \cite{bhagat2011node,rong2019dropedge,park2019estimating,zhang2018link,zhu2021neural,yun2021neo,errica2019fair,chen2019powerful}. While conventional approaches leverage existing edges in the original graph or user-defined constructions like radius or k-nearest neighbor graphs \cite{wu2020comprehensive}, they share a common limitation: these methods exclusively pass messages along positive relationships (assuming homophily), effectively discarding the potentially valuable information encoded in the absence of connections between nodes. For example, in a political network, the absence of collaborations or negative edges (signaling opposing ideological positions or party rivalries) between politicians can be just as informative as positive, existing collaborations. 

Several approaches have explored remediating this issue by learning over both positive and negative edges; they broadly fall into three categories: methods employing separate unconstrained weights for positive/negative edges \cite{liu2023pane,chen2024sigformer}, attention-based approaches that attend to all node pairs (with different sets of parameters for positive and negative edges) \cite{kreuzer2021rethinking}, and contrastive learning techniques \cite{you2020graph,zhu2021graph}. However, two key issues persist across these works. First, these works implicitly assume negative edges are always beneficial without characterization of when they may underperform. As a result, methods incorporating negative edges may exhibit poor performance in scenarios that appear random but follow predictable patterns based on graph properties. Second, existing approaches either entirely lack the potentially valuable geometric inductive bias---positively connected nodes should have similar embeddings and negatively connected nodes should have different embeddings---when reasoning over positive and negative relationships (e.g. unconstrained weight methods \cite{liu2023pane,chen2024sigformer}) or only apply it during training but not inference (e.g. methods using auxiliary contrastive losses \cite{you2020graph,zhu2021graph}).

As such, two core questions remain unanswered: \textbf{(Q1)} When and why are negative edges informative? \textbf{(Q2)} How can we integrate messages from positive and negative edges in a principled manner that let us effectively leverage the geometry of positive and negative edges?

\textbf{Contributions}. To answer the first question, we conduct an information-theoretic analysis on Stochastic Block Models \cite{holland1983stochastic} to quantify when negative edges provide meaningful signal. Our analysis reveals that negative edges deliver maximal information gain under the conditions of high homophily ratio, high edge density, and low label rates. Conversely, in graphs with low homophily, sparse connectivity, or abundant labeled data, negative edges contribute minimal additional information beyond what positive edges already provide. These findings establish both a theoretical justification for incorporating negative edges and clear guidance on the specific conditions where negative edge information will be beneficial—enabling practitioners to determine when negative edges should be employed and when they should be avoided.

To answer the second question, we introduce Contrastive Message Passing (CMP), a general message passing layer that reasons over both positive and negative edges by incorporating contrastive dynamics directly in the model architecture. By reweighting the negative eigenvalues of the weight matrices, CMP differentially transforms node features to bring positively connected nodes closer while pushing negatively connected nodes apart in a flexible and learnable manner. Our formulation offers two key advantages over loss-based contrastive methods: (1) CMP enables single-objective optimization focusing solely on the downstream task (e.g., node classification) rather than balancing multiple competing objectives (e.g., node classification + contrastive loss), and (2) CMP utilizes negative edge information during both training and inference, improving generalization. Furthermore, the modular design of CMP makes it easy to integrate it into existing GNN frameworks. We demonstrate this flexibility by implementing CMP variants for both GraphSage \cite{hamilton2017inductive} and Graph Attention Networks \cite{velivckovic2017graph}. 

Comprehensive evaluation across simulated environments and nine real-world datasets validates our theoretical findings and the CMP architecture. Primarily, CMP demonstrates significant improvements over standard GNNs (learning over only positive edges) on datasets with high homophily, edge density, and lower label rates—achieving up to $25.49\%$ average improvement at the smallest label rates and $14.78\%$ overall improvement against standard GNNs. When compared specifically against contrastive learning methods, CMP achieves a $22.95\%$ average improvement. On datasets without favorable characteristics, we observe that standard GNNs generally perform better, supporting the findings from our theoretical analysis. This pattern of results strongly validates that CMP effectively leverages negative edge information in precisely the scenarios where our theory predicts it would be most beneficial.

\section{Negative Edges are Informative}
\label{section3}
In this section, we demonstrate when negative edges (i.e., absent connections) contribute meaningful signal through information-theoretic analysis of Stochastic Block Models (SBMs) \cite{holland1983stochastic}. We quantify the conditions—specifically regarding homophily ratio, edge density, and label rate—under which negative edges yield information gain over positive edges alone. These findings will provide both theoretical justification for leveraging negative edges and practical guidance on when their use will likely improve model performance.

\textbf{Preliminaries}. Consider a graph $\mathcal{G} = (\mathcal{V}, \mathcal{E}^+, \mathcal{E}^-)$ with node set $\mathcal{V}$ (where $n = |\mathcal{V}|$), positive edge set $\mathcal{E}^+$, and negative edge set $\mathcal{E}^- = \{(i, j) \mid (i, j) \notin \mathcal{E}^+\}$.\footnote{In practice, $\mathcal{E}^-$ is constructed by uniformly sampling $|\mathcal{E}^+|$ edges from $\{(i, j) \mid (i, j) \notin \mathcal{E}^+\}$ to maintain linear complexity with the number of positive edges. All implementations in \Cref{experiments} follow this definition. } For each node $i \in \mathcal{V}$, we denote its positive neighborhood as $\mathcal{N}_{pos}(i) = \{j | (i,j) \in \mathcal{E}^+\}$ and its negative neighborhood as $\mathcal{N}_{neg}(i) = \{k | (i,k) \in \mathcal{E}^-\}$. $G$ is partitioned equally into $C$ communities where positive edges form with probability $p_{in}$ within communities and $p_{out}$ between communities. We characterize this graph through three key parameters: the homophily ratio $s = \frac{p_{in}}{p_{in} + (C-1)p_{out}}$ measuring the tendency for within-community connections, the label rate $r = \frac{m}{n}$ representing the fraction of nodes with known labels (where $m$ is the number of label nodes), and the edge density $d = \frac{p_{in} + (C-1)p_{out}}{C}$ capturing the overall graph connectivity. The learning task is node classification where we predict each node's community membership. To capture the diminishing value of additional labeled neighbors, we introduce a redundancy function $f: \mathbb{R}^+ \rightarrow [0,1]$ with the property that $f(x) \rightarrow 0$ as $x \rightarrow \infty$, where $x$ denotes the expected number of labeled neighbors.

\textbf{Analysis}. Consider an unlabeled node that exchanges messages with its neighbors. We characterize the information gain from neighboring positive or negative edges with the following theorem:
\begin{restatable}{theorem}{infotheorem}
\label{theorem1}
The expected information gain for an unlabeled node $i \in \mathcal{V}$ from its positive neighbors $\mathcal{N}_{pos}(i)$ is
\begin{align}
IG_{pos}(s,r,d) = nr d \Delta H^+(s,d) f_{pos}(s,r,d)
\end{align}
and the expected information gain for an unlabeled node $i \in \mathcal{V}$ from its negative neighbors $\mathcal{N}_{neg}(i)$ is
\begin{align}
IG_{neg}(s,r,d) = nr (1-d) \Delta H^-(s,d) f_{neg}(s,r,d)
\end{align}
where $\Delta H^+(s,d)$ and $\Delta H^-(s,d)$ are the information gains per positive and negative edge respectively given by \Cref{lemma1} and \Cref{lemma2}.
\end{restatable}
\Cref{theorem1} quantifies the absolute information gain from positive and negative edges across homophily, edge density, and label rate. While both $IG_{pos}$ and $IG_{neg}$ increase with $s$, $d$, and $r$, their \textit{rates of increase} differ significantly, creating specific regimes where negative edges are particularly valuable. To identify these regimes, we introduce the following corollary, which normalizes the contribution of negative edges relative to total information gain:
\begin{restatable}{corollary}{corr}
\label{corollary1}
For an unlabeled node $i \in \mathcal{V}$, the relative contribution of negative neighbors $\mathcal{N}_{neg}(i)$ compared to the total information gain (from all neighbors $\mathcal{N}_{pos}(i) \cup \mathcal{N}_{neg}(i)$) is
\begin{align}
R_{neg}(s,r,d) = \frac{(1-d) \Delta H^-(s,d) f_{neg}(s,r,d)}{d \Delta H^+(s,d) f_{pos}(s,r,d) + (1-d) \Delta H^-(s,d)f_{neg}(s,r,d)}.
\end{align}
\end{restatable}
\begin{figure}
  \centering
\includegraphics[width=0.9\textwidth]{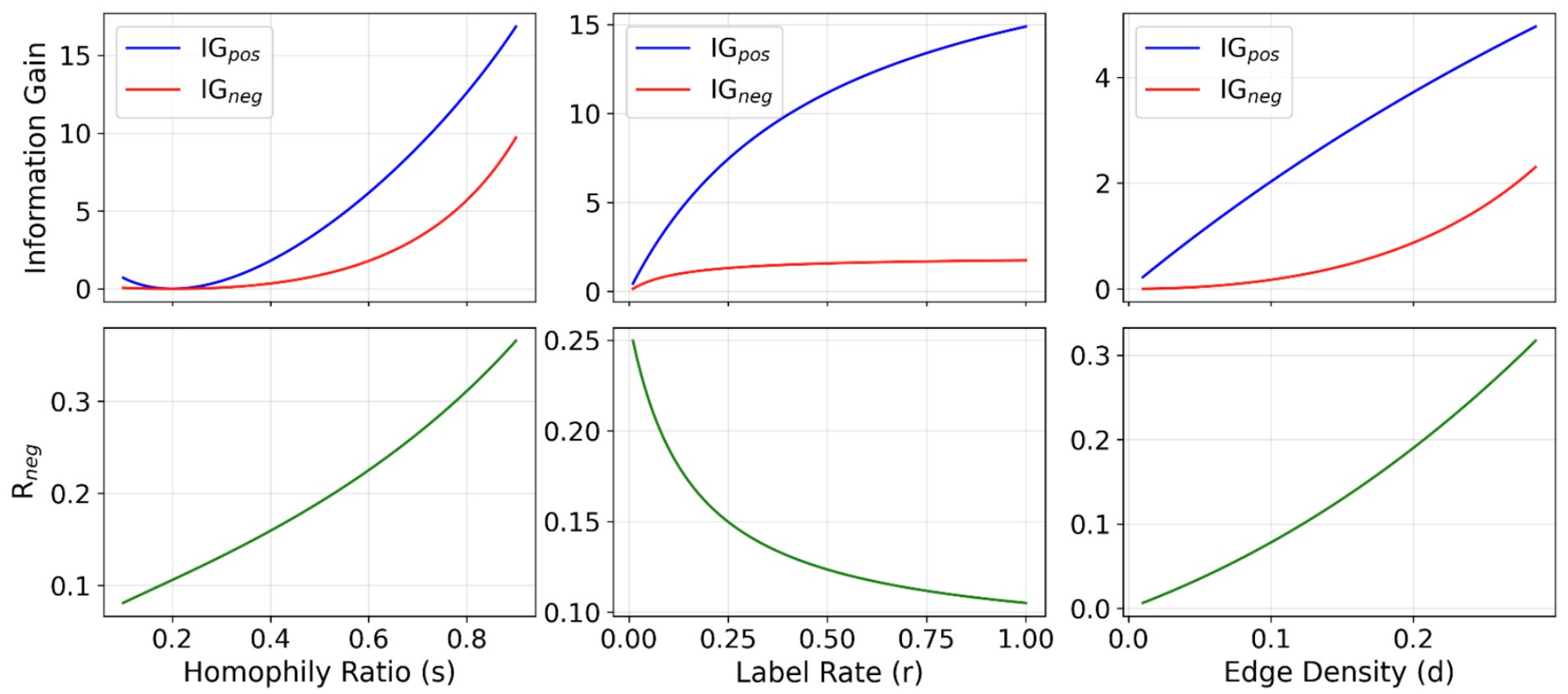}
  \caption{Values of (top row) $IG_{pos}$ and $IG_{neg}$ from \Cref{theorem1} and (bottom row) $R_{neg}$ from \Cref{corollary1} over different input levels of homophily, label rates, and edge densities.} 
  \label{theory_fig}
\end{figure}
We can interpret the implications of \Cref{theorem1} and \Cref{corollary1} by visualizing how information gain varies across the different graph parameters $s$, $r$, and $d$. We model information redundancy using the functional forms $f_{pos}(r,s,d) = \frac{1}{1+ \alpha (dsr)}$ and $f_{neg}(r,s,d) = \frac{1}{1+ \alpha ((1-d)(1-s)r)}$ where $\alpha = 20$ controls the rate of information saturation. These functions capture how information value diminishes as the expected number of labeled neighbors increases—$dsr$ for $\mathcal{E}^+$ and $(1-d)(1-s)r$ for $\mathcal{E}^-$. \Cref{theory_fig} presents the resulting information gains $IG_{pos}$, $IG_{neg}$, and the relative contribution $R_{neg}$ of negative edges across the parameter space.

\textbf{When Negative Edges are Informative}. Our analysis and \Cref{theory_fig} reveal three patterns governing when negative edges provide information gain: (\textbf{1. Homophily}) As homophily ratio $s$ increases, both $IG_{pos}$ and $IG_{neg}$ grow, but crucially, the relative contribution $R_{neg}$ also increases. This indicates that negative edges become disproportionately more valuable in highly homophilous graphs where community structure is clearer and the absence of edges carries stronger signal about node dissimilarity. (\textbf{2. Label Rate}) While both information gains increase with label rate $r$, the relative benefit $R_{neg}$ decreases. This inverse relationship suggests that negative edges are most valuable in sparsely labeled settings. As more labels become available, positive edges alone increasingly suffice for accurate classification. (\textbf{3. Edge Density}) Information gains and $R_{neg}$ all increase with edge density $d$. This pattern emerges because dense graphs reduce the likelihood of false negatives in random negative edge sampling; in sparse graphs, many unobserved edges may actually represent missing data rather than true negative relationships.

\textbf{Key Insight}. Negative edges provide maximum utility under conditions of \textit{high homophily} (strong community structure), \textit{high edge density} (reliable negative sampling), and \textit{low label rates} (limited supervised signal). In contrast, in graphs with weak community structure, sparse connectivity, or abundant labels, randomly sampled negative edges offer marginal additional information gain over positive edges.

Having established when negative edges can be informative, a key question is \textbf{how should graph neural networks differentially process positive versus negative edges?} While existing approaches employ separate unconstrained weight matrices for each edge type \cite{chen2024sigformer}, such methods lack the geometric inductive bias necessary for effective contrastive representation learning. We propose a principled solution based on the following core contrastive principle:
\begin{tcolorbox}[halign=center, left=2mm, top=1mm, right=2mm, bottom=1mm]
\label{principle_one}
\textbf{Principle 1}: Transform node embeddings during message passing such that positively connected nodes converge in representation space while negatively connected nodes diverge
\end{tcolorbox}
Our approach fundamentally differs from existing methods by embedding this contrastive principle \textit{directly into the message passing architecture}. In \Cref{experiments}, we empirically demonstrate the advantages of architecturally enforcing this contrastive logic rather than enforcing it via an auxiliary loss.

\section{Contrastive Message Passing}
\begin{figure}
  \centering
\includegraphics[width=0.9\textwidth]{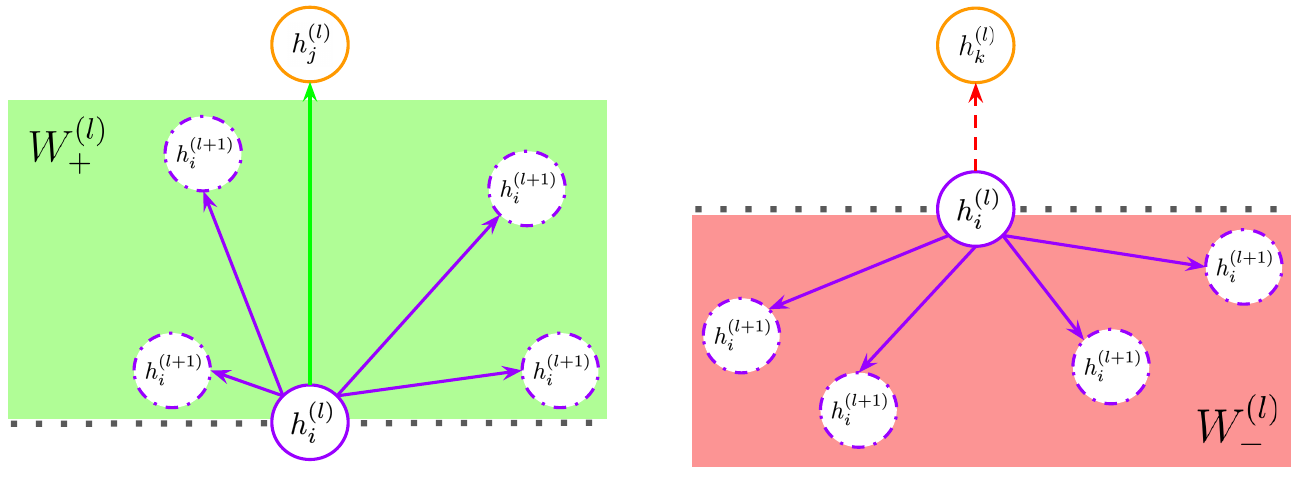}
  \caption{Geometric intuition of positive semidefinite weight (PSD) constraints. For positive edges (left), the PSD constraint on $W_{+}^{(l)}$ ensures that messages from neighbor $h_{j}^{(l)}$ are transformed to lie within the green region, which guarantees that the updated embedding $h_{i}^{(l+1)}$ moves closer to $h_{j}^{(l)}$ in representation space. For negative edges (right), the negated PSD matrix $-W_{-}^{(l)}$ constrains transformed messages from $h_{k}^{(l)}$ to the red region, ensuring $h_{i}^{(l+1)}$ moves away from $h_{k}^{(l)}$.}
  \label{psd_fig}
\end{figure}

Contrastive Message Passing (CMP) incorporates both positive and negative edges into the message passing scheme through separate transformations that encode the geometric inductive bias specified in Principle 1. We formalize this principle through the following message passing scheme:
\begin{align}
\label{cmp_equation}
    h_i^{(l+1)} = \text{agg}_{j \in \mathcal{N}_{pos}(i) \cup \{i\}} \left( W_{+}^{(l)}h_j^{(l)}\right) - \text{agg}_{k \in \mathcal{N}_{neg}(i)} \left( W_{-}^{(l)}h_k^{(l)}\right) \quad \text{s.t.}\quad  W_{+}^{(l)}, W_{-}^{(l)} \succeq 0
\end{align}
where $h_i^{(l)}$ is the node embedding at layer $l$, $\text{agg}$ denotes an aggregation function (e.g., mean or sum), $W_{+}^{(l)}$ and $W_{-}^{(l)}$ represent the transformation matrices for positive and negative edges respectively, and the notation $\succeq 0$ indicates positive semi-definiteness. The positive semi-definite (PSD) constraint on both weight matrices is critical to enforce the contrastive behavior of Principle 1: for positive edges, the transformation $W_{+}^{(l)}h_j^{(l)}$ maps neighbor features into the same halfspace as the original feature $h_j^{(l)}$. After aggregation, the updated embedding $h_i^{(l+1)}$ moves closer to its positive neighbors. For negative edges, while $W_{-}^{(l)}$ applies the same halfspace constraint, the negative sign in \Cref{cmp_equation} effectively flips this halfspace, ensuring that $h_i^{(l+1)}$ moves away from its negative neighbors. Figure \ref{psd_fig} illustrates this geometric interpretation. Note that while this is the most basic formulation for intuition, we provide specific implementations of CMP in popular architectures in \Cref{cmp_implementations}.

\subsection{\textit{Soft} Positive Semi-Definite Constraints}
\label{soft_psd_section}
\begin{figure}
  \centering
\includegraphics[width=0.9\textwidth]{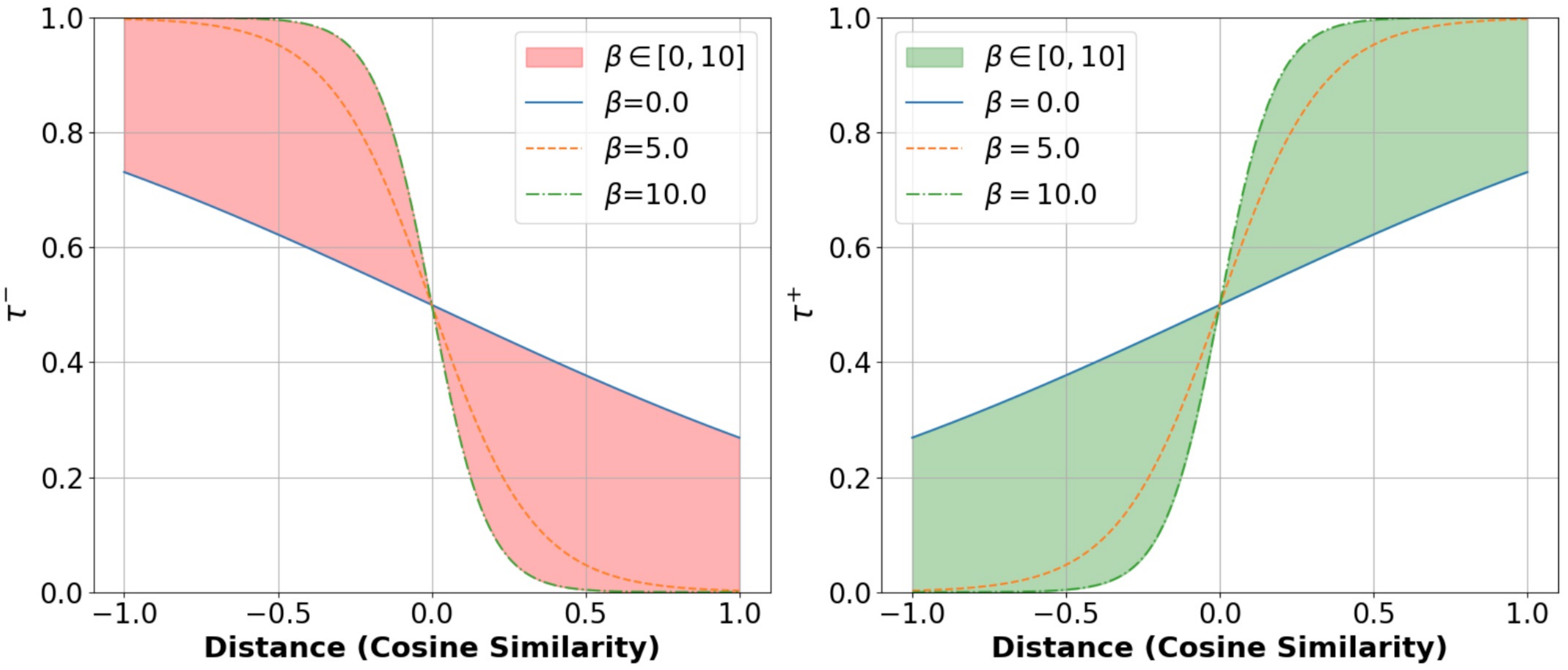}
  \caption{The $\tau$ value downweights the negative eigenvalues of the weight matrix based on the distance between nodes in the negative (left) and positive (right) case.} 
  \label{tau_beta_fig}
\end{figure}
While strict PSD constraints provide strong geometric inductive bias, they can be overly restrictive in practice. For instance, suppose two nodes that are positively connected are very similar. In this case, the nodes do not require being forced to be even closer together and can have some more flexibility with the learned transformation. Similarly, for negatively connected nodes, if the node embeddings are already very different, we do not need to continue pulling them apart by enforcing the PSD constraint. This relates to the common issue of collapsed representations in contrastive learning literature \cite{jing2021understanding}. To provide more learnable flexibility in these kinds of scenarios, we introduce soft PSD constraints that modulate the constraint strength based on node similarity. For any real symmetric matrix $W$ (i.e. the learnable weights), we apply:
\begin{align}
&\text{Soft-PSD}(W) = Q\hat{\Lambda}Q^{-1}
\quad \text{where} \quad \hat{\Lambda}_{ii} = 
\begin{cases}
\Lambda_{ii} & \text{if } \Lambda_{ii} \geq 0 \\
\tau \Lambda_{ii} & \text{if } \Lambda_{ii} < 0
\end{cases}.
\end{align}
The downweighting factor $\tau$ for negative eigenvalues is computed as:
\begin{align}
\tau = \begin{cases}
\sigma\left(+c(h_i,h_j)(1 + \beta)\right) & \text{for } W = W_{+} \\
\sigma\left(-c(h_i,h_j)(1 + \beta)\right) & \text{for } W = W_{-}
\end{cases}
\end{align}
where $\sigma$ is the sigmoid function, $c(h_i,h_j)$ denotes cosine similarity between node embeddings, and $\beta > 0$ is a learnable scalar. We provide additional details on learning through the soft PSD constraint in \Cref{soft_psd_learning_algo} and reasoning behind the formulation of $\tau$ in \Cref{tau_design}.

This formulation preserves positive eigenvalues while adaptively downweighting negative ones. Geometrically, negative eigenvalues correspond to eigenvectors along which the transformation reflects or inverts the input space. In our context, this means that for a transformation with negative eigenvalues, some directions in the feature space get flipped to the opposite halfspace, violating the contrastive principle that positive neighbors should remain close (and that negative neighbors should remain far). By downweighting negative eigenvalues with factor $\tau$, we control the degree of this violation: when $\tau = 0$, we recover strict PSD constraints of \Cref{cmp_equation}, preventing any reflections or flips over the halfspace, and when $\tau = 1$, negative eigenvalues remain unchanged, allowing arbitrary linear transformations. For negative edges, $\tau$ decreases as nodes become more similar, increasing the repulsive force. For positive edges, $\tau$ increases with similarity, relaxing constraints to enable more flexible (non-PSD) transformations of already similar nodes. \Cref{soft_psd_fig} visualizes this geometric intuition. 

The learnable parameter $\beta$ allows the model to adjust the sensitivity of this adaptation during training. \Cref{tau_beta_fig} illustrates how $\tau$ varies with node similarity for different $\beta$ values. This soft constraint effectively prevents representation collapse while maintaining the contrastive inductive bias. Throughout the remainder of this work, all CMP implementations use this soft PSD constraint.

\subsection{General Applicability and Design Principles} 
CMP holds several key properties that enable its simple integration with existing architectures: (\textbf{Modularity}) CMP can be incorporated as a drop-in augmentation for standard message passing layers, requiring only the addition of the negative edge aggregation term and the soft PSD constraint. (\textbf{Task Agnostic}) Unlike many contrastive learning approaches that are primarily designed for link prediction, CMP supports arbitrary downstream tasks such as node classification, graph classification, and regression problems. (\textbf{Architectural Compatibility}) CMP remains compatible with standard GNN operations including arbitrary activation functions, bias terms, normalization layers, and residual connections. However, we emphasize that while Principle 1 is specifically enforced at the message passing step, certain architectural choices (such as non-monotonically increasing activations or batch normalization) may modulate or even reverse this contrastive behavior. Rather than attempting to maintain contrastive relationships throughout the entire network, CMP provides this inductive bias at the message passing stage to effectively combine positive and negative edges while allowing subsequent operations to transform representations as needed for the downstream task.

To demonstrate the ease of integration, we present CMP-extended versions of two canonical architectures: GraphSAGE \cite{hamilton2017inductive} and Graph Attention Networks (GAT) \cite{velivckovic2017graph} with implementation details in \Cref{cmp_implementations}. 

\section{Experiments}
\label{experiments}
\begin{figure}
  \centering
\includegraphics[width=0.86\textwidth]{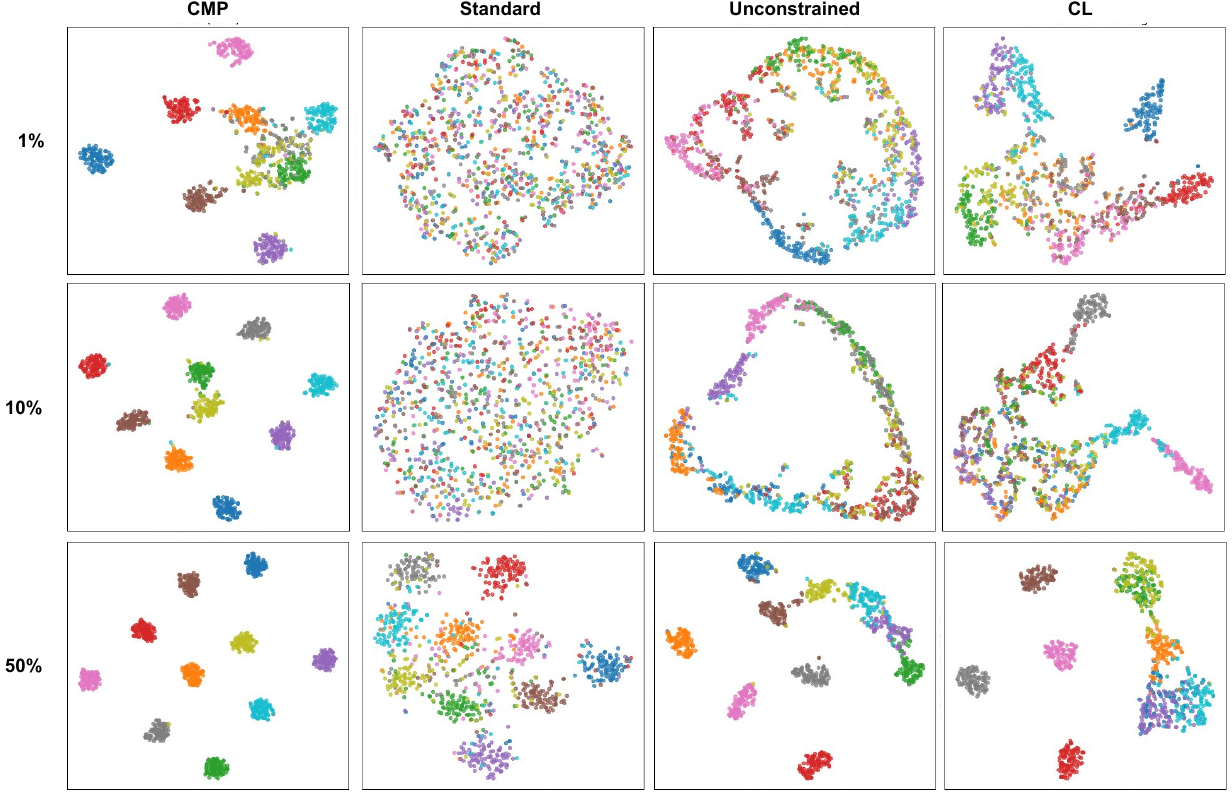}
  \caption{First two t-SNE dimensions of the Stochastic Block Model's learned embeddings taken from the last message passing layer of GraphSage. In each row are the node embeddings at a specific label rate.} 
  \label{sbm_embeds}
\end{figure}
\begin{figure}
  \centering
\includegraphics[width=\textwidth]{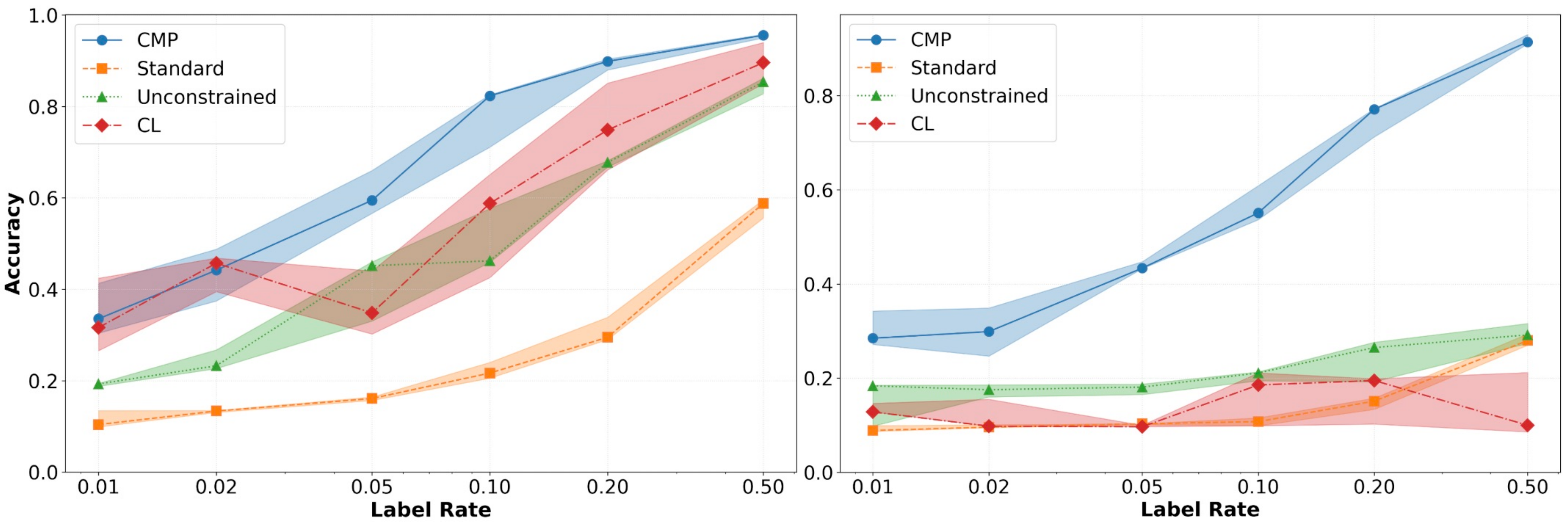}
  \caption{Class prediction accuracy of CMP and baselines on the Stochastic Block Model over different label rates on the (left) GAT-CMP and (right) GraphSage-CMP architectures. The line and shaded regions show the median and $25$/$75$th percentile values, respectively.} 
  \label{sbm_perf}
\end{figure}

We evaluate CMP across synthetic datasets (Stochastic Block Model) and 9 real-world graph datasets including citation networks (Cora, CiteSeer, PubMed) \cite{yang2016revisiting}, co-authorship networks (CS, Physics) \cite{shchur2018pitfalls}, Amazon co-purchase networks (Computers, Photo) \cite{shchur2018pitfalls}, a social network (Flickr \cite{zeng2019graphsaint}), and biological networks (PPI \cite{zitnik2017predicting}). These datasets exhibit varying homophily ratios and edge densities---factors that influence the usefulness of negative edges as shown in \Cref{section3}. We systematically vary label rates from low to high to investigate how each method performs with different levels of supervision. Negative edges are randomly sampled between unconnected nodes, matching the number of positive edges ($1$:$1$ ratio). All results show median performance and/or $25$/$75$th percentiles over $5$ runs with different seeds.

\textbf{Baselines.} We evaluate CMP against three baseline categories: (1) Standard GNN with message passing only along positive edges, (2) Unconstrained GNN with message passing along both positive and negative edges without soft PSD constraints, and (3) Contrastive Learning (CL) GNN, which is a standard GNN with an auxiliary contrastive loss following \Cref{cl_equation}. These prototypical baselines deliberately focus on fundamental architectural differences rather than implementation-specific optimizations, allowing us to isolate how negative edges and the proposed PSD constraints contribute across different settings.

\textbf{Implementation.} All models share the same architecture: a linear lifting layer, two message passing layers, and a linear projection head with LeakyReLU activations. The message passing follows \Cref{cmp_implementations} for GraphSAGE and GAT implementations. Additional details are in \Cref{app_model_details}.

\subsection{Stochastic Block Model}
\label{sbm_experiments}
\textbf{Setup.} We construct a Stochastic Block Model (SBM) with $n=1000$ nodes uniformly distributed across $C=10$ communities. Each node has $d=32$ dimensional features initialized randomly from $\mathcal{N}(0,1)$. The intra-community and inter-community edge probabilities are $p_{in}=0.25$ and $p_{out}=0.05$, respectively. We evaluate performance across label rates $r \in \{1\%, 2\%, 5\%, 10\%, 20\%, 50\%\}$ to examine behavior in different supervision regimes. Furthermore, we evaluate CMP and baselines at different levels of heterophily, which is included in \Cref{heterophily_sbm_exp}.

\textbf{Results.} \Cref{sbm_perf} shows that CMP consistently outperforms baseline methods across different label rates and architectures. While baseline methods show architecture-dependent performance (with GAT utilizing additional labels more effectively than GraphSAGE), CMP achieves strong performance with both architectures. 

The improvement in CMP performance can be explained with \Cref{sbm_embeds}, which visualizes the learned node embeddings of different methods. At $r=1\%$, CMP already produces well-separated community clusters while baseline methods struggle to distinguish communities. As label rates increase, all methods eventually achieve community separation, but CMP maintains well-delineated clusters throughout. This empirically validates that the soft PSD constraint used in CMP is effective in enforcing contrastive representations of samples from shared or different communities, especially in low-label rate regimes as expected given the findings of \Cref{section3}.

\subsection{Real Datasets}
\label{real_experiments}
\begin{table}[t]
\centering
\caption{Median test accuracy performance of GraphSage across real datasets and label rates (\% labels used from total labels). Bold values indicate the best performing model.}
\label{real_data_acc_results}
\resizebox{\textwidth}{!}{
\definecolor{lightgray}{gray}{0.9}
\definecolor{blue}{rgb}{0.7, 0.8, 1.0}
\definecolor{orange}{rgb}{1.0, 0.8, 0.6}
\definecolor{purple}{rgb}{0.9, 0.7, 0.9}
\begin{tabular}{l|l|c|c|c|c|c|c|c|c|c}
\toprule
Dataset & Model & 0.1\% & 0.2\% & 0.5\% & 1\% & 2\% & 5\% & 10\% & 20\% & 50\% \\
\midrule
Cora & CMP & \cellcolor{lightgray}{--} & \cellcolor{lightgray}{--} & \cellcolor{lightgray}{--} & \cellcolor{purple}{63.96} & \cellcolor{purple}{\textbf{71.64}} & \cellcolor{purple}{\textbf{77.64}} & \cellcolor{purple}{\textbf{79.80}} & \cellcolor{blue}{\textbf{82.81}} & \cellcolor{blue}{\textbf{86.35}} \\
& Standard & \cellcolor{lightgray}{--} & \cellcolor{lightgray}{--} & \cellcolor{lightgray}{--} & \cellcolor{purple}{47.49} & \cellcolor{purple}{61.24} & \cellcolor{purple}{71.69} & \cellcolor{purple}{77.81} & \cellcolor{blue}{82.76} & \cellcolor{blue}{85.70} \\
& Unconstrained & \cellcolor{lightgray}{--} & \cellcolor{lightgray}{--} & \cellcolor{lightgray}{--} & \cellcolor{purple}{62.46} & \cellcolor{purple}{66.69} & \cellcolor{purple}{75.55} & \cellcolor{purple}{78.32} & \cellcolor{blue}{81.34} & \cellcolor{blue}{83.76} \\
& CL & \cellcolor{lightgray}{--} & \cellcolor{lightgray}{--} & \cellcolor{lightgray}{--} & \cellcolor{purple}{\textbf{67.98}} & \cellcolor{purple}{68.92} & \cellcolor{purple}{71.65} & \cellcolor{purple}{75.23} & \cellcolor{blue}{77.07} & \cellcolor{blue}{83.39} \\
\midrule
CiteSeer & CMP & \cellcolor{lightgray}{--} & \cellcolor{lightgray}{--} & \cellcolor{lightgray}{--} & \cellcolor{orange}{55.40} & \cellcolor{orange}{\textbf{61.63}} & \cellcolor{orange}{64.79} & \cellcolor{orange}{67.97} & \cellcolor{orange}{70.56} & \cellcolor{orange}{73.20} \\
& Standard & \cellcolor{lightgray}{--} & \cellcolor{lightgray}{--} & \cellcolor{lightgray}{--} & \cellcolor{orange}{52.13} & \cellcolor{orange}{59.85} & \cellcolor{orange}{\textbf{67.02}} & \cellcolor{orange}{\textbf{68.57}} & \cellcolor{orange}{\textbf{71.29}} & \cellcolor{orange}{\textbf{74.62}} \\
& Unconstrained & \cellcolor{lightgray}{--} & \cellcolor{lightgray}{--} & \cellcolor{lightgray}{--} & \cellcolor{orange}{55.77} & \cellcolor{orange}{60.81} & \cellcolor{orange}{64.26} & \cellcolor{orange}{65.64} & \cellcolor{orange}{68.11} & \cellcolor{orange}{70.65} \\
& CL & \cellcolor{lightgray}{--} & \cellcolor{lightgray}{--} & \cellcolor{lightgray}{--} & \cellcolor{orange}{\textbf{59.18}} & \cellcolor{orange}{60.05} & \cellcolor{orange}{63.66} & \cellcolor{orange}{66.99} & \cellcolor{orange}{68.15} & \cellcolor{orange}{72.00} \\
\midrule
PubMed & CMP & \cellcolor{orange}{64.97} & \cellcolor{orange}{71.02} & \cellcolor{orange}{74.67} & \cellcolor{orange}{76.08} & \cellcolor{orange}{78.40} & \cellcolor{orange}{80.72} & \cellcolor{lightgray}{--} & \cellcolor{lightgray}{--} & \cellcolor{lightgray}{--} \\
& Standard & \cellcolor{orange}{59.59} & \cellcolor{orange}{71.08} & \cellcolor{orange}{\textbf{77.74}} & \cellcolor{orange}{\textbf{79.89}} & \cellcolor{orange}{\textbf{81.21}} & \cellcolor{orange}{\textbf{83.62}} & \cellcolor{lightgray}{--} & \cellcolor{lightgray}{--} & \cellcolor{lightgray}{--} \\
& Unconstrained & \cellcolor{orange}{60.44} & \cellcolor{orange}{69.63} & \cellcolor{orange}{72.90} & \cellcolor{orange}{75.39} & \cellcolor{orange}{76.43} & \cellcolor{orange}{79.97} & \cellcolor{lightgray}{--} & \cellcolor{lightgray}{--} & \cellcolor{lightgray}{--} \\
& CL & \cellcolor{orange}{\textbf{69.38}} & \cellcolor{orange}{\textbf{73.66}} & \cellcolor{orange}{75.32} & \cellcolor{orange}{76.10} & \cellcolor{orange}{78.36} & \cellcolor{orange}{80.42} & \cellcolor{lightgray}{--} & \cellcolor{lightgray}{--} & \cellcolor{lightgray}{--} \\
\midrule
CS & CMP & \cellcolor{purple}{\textbf{67.10}} & \cellcolor{purple}{\textbf{77.09}} & \cellcolor{purple}{\textbf{85.65}} & \cellcolor{purple}{\textbf{86.50}} & \cellcolor{blue}{\textbf{89.47}} & \cellcolor{blue}{91.01} & \cellcolor{lightgray}{--} & \cellcolor{lightgray}{--} & \cellcolor{lightgray}{--} \\
& Standard & \cellcolor{purple}{60.52} & \cellcolor{purple}{70.40} & \cellcolor{purple}{81.08} & \cellcolor{purple}{83.42} & \cellcolor{blue}{88.60} & \cellcolor{blue}{\textbf{92.20}} & \cellcolor{lightgray}{--} & \cellcolor{lightgray}{--} & \cellcolor{lightgray}{--} \\
& Unconstrained & \cellcolor{purple}{65.53} & \cellcolor{purple}{76.14} & \cellcolor{purple}{83.95} & \cellcolor{purple}{85.13} & \cellcolor{blue}{88.48} & \cellcolor{blue}{90.08} & \cellcolor{lightgray}{--} & \cellcolor{lightgray}{--} & \cellcolor{lightgray}{--} \\
& CL & \cellcolor{purple}{63.90} & \cellcolor{purple}{69.52} & \cellcolor{purple}{77.35} & \cellcolor{purple}{78.50} & \cellcolor{blue}{82.06} & \cellcolor{blue}{83.67} & \cellcolor{lightgray}{--} & \cellcolor{lightgray}{--} & \cellcolor{lightgray}{--} \\
\midrule
Physics & CMP & \cellcolor{purple}{\textbf{88.51}} & \cellcolor{purple}{\textbf{90.92}} & \cellcolor{purple}{93.04} & \cellcolor{purple}{\textbf{94.05}} & \cellcolor{blue}{94.48} & \cellcolor{blue}{95.00} & \cellcolor{lightgray}{--} & \cellcolor{lightgray}{--} & \cellcolor{lightgray}{--} \\
& Standard & \cellcolor{purple}{84.07} & \cellcolor{purple}{86.53} & \cellcolor{purple}{91.13} & \cellcolor{purple}{94.03} & \cellcolor{blue}{\textbf{94.97}} & \cellcolor{blue}{\textbf{95.56}} & \cellcolor{lightgray}{--} & \cellcolor{lightgray}{--} & \cellcolor{lightgray}{--} \\
& Unconstrained & \cellcolor{purple}{83.46} & \cellcolor{purple}{89.78} & \cellcolor{purple}{\textbf{93.27}} & \cellcolor{purple}{93.75} & \cellcolor{blue}{94.27} & \cellcolor{blue}{94.79} & \cellcolor{lightgray}{--} & \cellcolor{lightgray}{--} & \cellcolor{lightgray}{--} \\
& CL & \cellcolor{purple}{83.95} & \cellcolor{purple}{86.89} & \cellcolor{purple}{90.82} & \cellcolor{purple}{92.32} & \cellcolor{blue}{93.50} & \cellcolor{blue}{94.26} & \cellcolor{lightgray}{--} & \cellcolor{lightgray}{--} & \cellcolor{lightgray}{--} \\
\midrule
Computers & CMP & \cellcolor{purple}{\textbf{59.41}} & \cellcolor{purple}{\textbf{71.40}} & \cellcolor{purple}{\textbf{77.94}} & \cellcolor{purple}{\textbf{81.67}} & \cellcolor{blue}{\textbf{85.50}} & \cellcolor{blue}{\textbf{87.22}} & \cellcolor{lightgray}{--} & \cellcolor{lightgray}{--} & \cellcolor{lightgray}{--} \\
& Standard & \cellcolor{purple}{44.06} & \cellcolor{purple}{53.97} & \cellcolor{purple}{66.33} & \cellcolor{purple}{76.52} & \cellcolor{blue}{81.31} & \cellcolor{blue}{83.64} & \cellcolor{lightgray}{--} & \cellcolor{lightgray}{--} & \cellcolor{lightgray}{--} \\
& Unconstrained & \cellcolor{purple}{53.99} & \cellcolor{purple}{64.37} & \cellcolor{purple}{76.95} & \cellcolor{purple}{80.00} & \cellcolor{blue}{84.58} & \cellcolor{blue}{86.55} & \cellcolor{lightgray}{--} & \cellcolor{lightgray}{--} & \cellcolor{lightgray}{--} \\
& CL & \cellcolor{purple}{42.11} & \cellcolor{purple}{47.04} & \cellcolor{purple}{38.43} & \cellcolor{purple}{37.39} & \cellcolor{blue}{37.34} & \cellcolor{blue}{37.50} & \cellcolor{lightgray}{--} & \cellcolor{lightgray}{--} & \cellcolor{lightgray}{--} \\
\midrule
Photo & CMP & \cellcolor{purple}{52.84} & \cellcolor{purple}{\textbf{68.66}} & \cellcolor{purple}{\textbf{81.54}} & \cellcolor{purple}{\textbf{89.03}} & \cellcolor{blue}{\textbf{90.83}} & \cellcolor{blue}{91.27} & \cellcolor{lightgray}{--} & \cellcolor{lightgray}{--} & \cellcolor{lightgray}{--} \\
& Standard & \cellcolor{purple}{32.03} & \cellcolor{purple}{42.77} & \cellcolor{purple}{69.20} & \cellcolor{purple}{78.48} & \cellcolor{blue}{87.72} & \cellcolor{blue}{\textbf{91.90}} & \cellcolor{lightgray}{--} & \cellcolor{lightgray}{--} & \cellcolor{lightgray}{--} \\
& Unconstrained & \cellcolor{purple}{\textbf{56.32}} & \cellcolor{purple}{63.01} & \cellcolor{purple}{80.97} & \cellcolor{purple}{88.06} & \cellcolor{blue}{89.23} & \cellcolor{blue}{90.77} & \cellcolor{lightgray}{--} & \cellcolor{lightgray}{--} & \cellcolor{lightgray}{--} \\
& CL & \cellcolor{purple}{54.59} & \cellcolor{purple}{56.07} & \cellcolor{purple}{59.30} & \cellcolor{purple}{50.23} & \cellcolor{blue}{63.65} & \cellcolor{blue}{49.78} & \cellcolor{lightgray}{--} & \cellcolor{lightgray}{--} & \cellcolor{lightgray}{--} \\
\midrule
PPI & CMP & \cellcolor{purple}{\textbf{45.83}} & \cellcolor{purple}{\textbf{45.31}} & \cellcolor{purple}{\textbf{47.75}} & \cellcolor{purple}{48.60} & \cellcolor{blue}{51.15} & \cellcolor{blue}{50.77} & \cellcolor{lightgray}{--} & \cellcolor{lightgray}{--} & \cellcolor{lightgray}{--} \\
& Standard & \cellcolor{purple}{44.80} & \cellcolor{purple}{45.05} & \cellcolor{purple}{46.99} & \cellcolor{purple}{\textbf{49.85}} & \cellcolor{blue}{\textbf{51.66}} & \cellcolor{blue}{\textbf{55.75}} & \cellcolor{lightgray}{--} & \cellcolor{lightgray}{--} & \cellcolor{lightgray}{--} \\
& Unconstrained & \cellcolor{purple}{44.91} & \cellcolor{purple}{43.86} & \cellcolor{purple}{46.91} & \cellcolor{purple}{48.57} & \cellcolor{blue}{48.77} & \cellcolor{blue}{50.94} & \cellcolor{lightgray}{--} & \cellcolor{lightgray}{--} & \cellcolor{lightgray}{--} \\
& CL & \cellcolor{purple}{42.10} & \cellcolor{purple}{41.54} & \cellcolor{purple}{43.55} & \cellcolor{purple}{43.58} & \cellcolor{blue}{44.36} & \cellcolor{blue}{44.84} & \cellcolor{lightgray}{--} & \cellcolor{lightgray}{--} & \cellcolor{lightgray}{--} \\
\midrule
Flickr & CMP & \cellcolor{orange}{42.51} & \cellcolor{orange}{\textbf{43.06}} & \cellcolor{orange}{\textbf{43.60}} & \cellcolor{orange}{43.96} & \cellcolor{orange}{44.36} & \cellcolor{orange}{\textbf{45.70}} & \cellcolor{lightgray}{--} & \cellcolor{lightgray}{--} & \cellcolor{lightgray}{--} \\
& Standard & \cellcolor{orange}{42.44} & \cellcolor{orange}{42.56} & \cellcolor{orange}{43.50} & \cellcolor{orange}{\textbf{44.01}} & \cellcolor{orange}{44.36} & \cellcolor{orange}{45.43} & \cellcolor{lightgray}{--} & \cellcolor{lightgray}{--} & \cellcolor{lightgray}{--} \\
& Unconstrained & \cellcolor{orange}{42.37} & \cellcolor{orange}{42.63} & \cellcolor{orange}{43.50} & \cellcolor{orange}{43.67} & \cellcolor{orange}{44.51} & \cellcolor{orange}{45.24} & \cellcolor{lightgray}{--} & \cellcolor{lightgray}{--} & \cellcolor{lightgray}{--} \\
& CL & \cellcolor{orange}{\textbf{43.04}} & \cellcolor{orange}{42.83} & \cellcolor{orange}{43.31} & \cellcolor{orange}{43.94} & \cellcolor{orange}{\textbf{44.61}} & \cellcolor{orange}{45.02} & \cellcolor{lightgray}{--} & \cellcolor{lightgray}{--} & \cellcolor{lightgray}{--} \\
\bottomrule
\end{tabular}
}
\end{table}
\textbf{Setup}. We evaluate on $9$ real-world datasets composed of $8$ transductive learning tasks and $1$ inductive task (PPI) with varying graph sizes: small datasets (Cora, CiteSeer) use label rates from $1$-$50\%$, while large datasets use $0.1$-$5\%$. \Cref{real_dataset_graph_table} reports important graph characteristics (e.g. homophily ratio and edge density) that influence the usefulness of negative edges based on \Cref{section3}. GraphSage results are reported here and GAT results in \Cref{gat_results_real_datasets}; both architectures exhibit consistent patterns. 

\textbf{Results}. \Cref{real_data_acc_results} presents the performance across datasets and label rates. We highlight three takeaways:

\textbf{(1) Performance Across Label Rates}: CMP generally outperforms all baselines in low-label regimes, shown in $\color{purple}{\text{purple}}$. However, at high label regimes given by $\color{blue}{\text{blue}}$, we see that often the standard baseline performs best or has comparable performance to CMP. This difference in performance between CMP and the standard baseline reinforces the insight from \Cref{theory_fig}: as label rates increase, the relative information gain from negative edges diminishes since positive edges alone provide sufficient signal for accurate predictions. 

\textbf{(2) Performance Across Graph Characteristics}: CMP performs similarly against baselines on datasets shown in $\color{orange}{\text{orange}}$. To understand why, we compute the graph characteristics of each dataset in \Cref{real_dataset_graph_table}. We find that the performance advantages of CMP does not extend to datasets with low homophily ratios, sparse edge densities, and low cluster coefficients. This empirically reaffirms our findings from \Cref{section3}, suggesting that negative edges do not indeed provide a meaningful information gain in real graphs lacking strong community structure.

\textbf{(3) Performance of CMP versus Negative Edge Baselines}: CMP consistently outperforms both the unconstrained baseline and contrastive learning baseline across most datasets and label rates. These performance improvements can be attributed to the importance of the soft PSD constraint (CMP vs. Unconstrained) and having access to negative edges at inference time (CMP vs. CL). 

\section{Related Works}
\label{background_section}
\textbf{Graph Learning with Negative Edges}. Traditional graph neural networks typically operate exclusively on existing or positive edges, but several works have recognized the potential value of nonexistent or negative edges. In recommender systems, \cite{liu2023pane,chen2024sigformer} incorporate negative edges through separate weight matrices, while attention-based models \cite{yun2019graph,kreuzer2021rethinking} implicitly consider negative connections by allowing attention across all node pairs. However, these approaches employ unconstrained weights without a principled inductive bias and lack rigorous analysis of when negative edges provide information gain. Our work advances beyond these limitations by demonstrating scenarios where negative edges offer information gain and providing a principled architecture to handle positive versus negative edges differentially.

\textbf{Contrastive Learning}. There are connections between our work and contrastive learning techniques that have been applied to graph settings \cite{you2020graph,zhu2021graph,xu2021infogcl} after its initial development in image recognition \cite{chopra2005learning}. The common formulation typically maximizes embedding similarity between positive pairs while minimizing embedding similarity between negative pairs through a contrastive objective \cite{hamilton2017inductive}:
\begin{align}
\label{cl_equation}
    \mathcal{L}_{\mathcal{G}}(h_i) &= -\log \left(\sigma(h_i^\top h_j)\right) - Q \cdot \mathbb{E}_{z \sim P_n(i)} \log \left(\sigma(-h_i^\top h_{z})\right)
\end{align}
where $j$ is a positive node to $i$, $\sigma$ is the sigmoid function, $P_n$ is a negative sampling distribution, and $Q$ defines the number of negative samples. The core difference with our method is that we directly integrate this contrastive principle into the message passing architecture rather than as an auxiliary loss. This approach offers two benefits: we only need to optimize a single downstream objective instead of a complex multi-objective loss, and we have access to negative edges at test time, enabling better generalization capabilities, which was shown in \Cref{experiments}.

\textbf{Negative Sampling}. Sampling informative negative edges is itself a challenging, domain-dependent problem. In an ideal setting, one would draw negatives directly from domain knowledge or ground-truth data—for instance, explicitly recorded “dislikes” in recommender systems \cite{liu2023pane}. In practice, however, most graphs arise in a positive–unlabeled context \cite{li2003learning,elkan2008learning,kiryo2017positive}, where only confirmed positive edges are recorded and the vast majority of potential edges are unlabelled rather than truly negative. The work in \cite{prakash2024evaluating} shows that the choice of negative-sampling strategy can bias personalization models, highlighting the importance of this design decision.
To address the positive–unlabeled setting, prior work has proposed hard-negative sampling \cite{robinson2020contrastive,schroff2015facenet} and random sampling with debiasing corrections \cite{chuang2020debiased,du2015convex}. Although these techniques can improve learning signals, they often introduce substantial computational overhead and require meticulous hyper-parameter tuning. As a result, the field of graph learning still relies mainly on simpler strategies: uniform random sampling from non-connected node pairs \cite{hamilton2017inductive} and structured negative sampling that respects graph topology or metadata \cite{kanatsoulis2021tex}.

\section{Conclusion}
We introduced Contrastive Message Passing (CMP), a message passing architecture that leverages both positive (aka existing) and negative (aka non-existent) edges in graph neural networks by imposing spectral constraints on learnable weights. Through theoretical analysis, we showed that negative edges provide maximum information gain under the specific conditions of high homophily, high edge density, and low label rates. By directly embedding contrastive learning dynamics into the message passing architecture rather than using auxiliary losses, CMP enables single-objective optimization and utilizes negative edges during both training and inference. Future work could explore non-random negative sampling strategies to improve performance in challenging scenarios (low homophily, sparse connectivity, high label rates) and develop computationally efficient alternatives to enforcing soft PSD constraints without the eigendecomposition operation.

\begin{ack}
We thank Vijay Prakash Dwivedi for helpful feedback and discussions. PPH was supported in part
by the NSF GRFP. We also acknowledge the support of NSF under Nos. OAC-1835598 (CINES), CCF-1918940 (Expeditions), DMS-2327709 (IHBEM), IIS-2403318 (III); Stanford Data Applications Initiative, Wu Tsai Neurosciences Institute, Stanford Institute for Human-Centered AI, Chan Zuckerberg Initiative, Amazon, Genentech, GSK, Hitachi, SAP, and UCB.
\end{ack}

\bibliography{neurips_2026}

\appendix

\section{Proofs}
\label{app_proofs}

\begin{restatable}{lemma}{informationgainpos}
\label{lemma1}
Suppose $p_{\text{in}}$ and $p_{\text{out}}$ are intra and inter community edge probability, respectively. The information gain from observing a positive edge (existing connection) between an unlabeled node $i$ and a labeled node $j$ with label $Y_j = c_j$ is
\begin{align}
\Delta H^+_i &= \log C + h^+ \log h^+ + (1-h^+) \log \frac{1-h^+}{C-1}
\end{align}
where
\begin{align}
    h^+ &= \frac{p_{in}}{p_{in} + (C-1)p_{out}}
\end{align}
\end{restatable}
\begin{proof}
Without prior information, node $i$ has uniform probability of belonging to any community
\begin{align}
P(Y_i = c) = \frac{1}{C} \text{ for all } c \in \{1,2,...,C\}
\end{align}

The entropy of this uniform distribution is:
\begin{align}
H(Y_i) &= -\sum_{c=1}^{C} P(Y_i = c) \log P(Y_i = c)\\
&= -\sum_{c=1}^{C} \frac{1}{C} \log \frac{1}{C}\\
&= \log C
\end{align}
When we observe an edge between nodes $i$ and $j$, and know node $j$'s label ($Y_j = c_j$), we apply Bayes' rule
\begin{align}
P(Y_i = c | Y_j = c_j, A_{ij} = 1) = \frac{P(A_{ij} = 1 | Y_i = c, Y_j = c_j) P(Y_i = c | Y_j = c_j)}{P(A_{ij} = 1 | Y_j = c_j)}
\end{align}
Given SBM assumptions and uniform prior on $Y_i$,
\begin{align}
P(A_{ij} = 1 | Y_i = c, Y_j = c_j) &= 
\begin{cases}
p_{in} & \text{if } c = c_j\\
p_{out} & \text{if } c \neq c_j
\end{cases}\\
P(Y_i = c | Y_j = c_j) &= \frac{1}{C} \text{ (uniform prior)}\\
P(A_{ij} = 1 | Y_j = c_j) &= \frac{1}{C}(p_{in} + (C-1)p_{out})
\end{align}
This gives us
\begin{align}
P(Y_i = c_j | Y_j = c_j, A_{ij} = 1) &= \frac{p_{in}}{p_{in} + (C-1)p_{out}} = h^+\\
P(Y_i = c | Y_j = c_j, A_{ij} = 1) &= \frac{p_{out}}{p_{in} + (C-1)p_{out}} = \frac{1-h^+}{C-1} \text{ for } c \neq c_j.
\end{align}
Then, we compute the conditional entropy as
\begin{align}
H(Y_i | Y_j = c_j, A_{ij} = 1) &= -h^+ \log h^+ - \sum_{c \neq c_j} \frac{1-h^+}{C-1} \log \frac{1-h^+}{C-1}\\
&= -h^+ \log h^+ - (1-h^+) \log \frac{1-h^+}{C-1}
\end{align}
Consequently,
\begin{align}
\Delta H^+_i &= H(Y_i) - H(Y_i | Y_j = c_j, A_{ij} = 1)\\
&= \log C + h \log h^+ + (1-h^+) \log \frac{1-h^+}{C-1}
\end{align}
\end{proof}

\begin{restatable}{lemma}{informationgainneg}
\label{lemma2}
    Suppose $p_{\text{in}}$ and $p_{\text{out}}$ are intra and inter community edge probability, respectively. The information gain from observing a negative edge (non-connection) between an unlabeled node $i$ and a labeled node $j$ with label $Y_j = c_j$ is
\begin{align}
\Delta H^-_i &= \log C + h^- \log h^- + (1-h^-) \log \frac{1-h^-}{C-1}
\end{align}
where
\begin{align}
h^- &= \frac{1-p_{in}}{(1-p_{in}) + (C-1)(1-p_{out})}
\end{align}
\end{restatable}
\begin{proof}
The proof structure will follow very similarly to \Cref{lemma1}, but we repeat here for completeness. Without prior information, suppose node $i$ has uniform probability of belonging to any community
\begin{align}
P(Y_i = c) = \frac{1}{C} \text{ for all } c \in \{1,2,...,C\}.
\end{align}
The entropy of this uniform distribution is
\begin{align}
H(Y_i) &= -\sum_{c=1}^{C} P(Y_i = c) \log P(Y_i = c)\\
&= -\sum_{c=1}^{C} \frac{1}{C} \log \frac{1}{C}\\
&= \log C
\end{align}
When we observe no edge between nodes $i$ and $j$, and know node $j$'s label ($Y_j = c_j$), we apply Bayes' rule to compute the conditional
\begin{align}
P(Y_i = c | Y_j = c_j, A_{ij} = 0) = \frac{P(A_{ij} = 0 | Y_i = c, Y_j = c_j) P(Y_i = c | Y_j = c_j)}{P(A_{ij} = 0 | Y_j = c_j)}.
\end{align}
Given SBM assumptions and uniform prior on $Y_i$,
\begin{align}
P(A_{ij} = 0 | Y_i = c, Y_j = c_j) &= 
\begin{cases}
1-p_{in} & \text{if } c = c_j\\
1-p_{out} & \text{if } c \neq c_j
\end{cases}\\
P(Y_i = c | Y_j = c_j) &= \frac{1}{C} \text{ (uniform prior)}\\
P(A_{ij} = 0 | Y_j = c_j) &= \frac{1}{C}((1-p_{in}) + (C-1)(1-p_{out})),
\end{align}
which results in
\begin{align}
P(Y_i = c_j | Y_j = c_j, A_{ij} = 0) &= \frac{1-p_{in}}{(1-p_{in}) + (C-1)(1-p_{out})} = h^-\\
P(Y_i = c | Y_j = c_j, A_{ij} = 0) &= \frac{1-p_{out}}{(1-p_{in}) + (C-1)(1-p_{out})} = \frac{1-h^-}{C-1} \text{ for } c \neq c_j.
\end{align}
Then, we compute the conditional entropy as
\begin{align}
H(Y_i | Y_j = c_j, A_{ij} = 0) &= -h^- \log h^- - \sum_{c \neq c_j} \frac{1-h^-}{C-1} \log \frac{1-h^-}{C-1}\\
&= -h^- \log h^- - (1-h^-) \log \frac{1-h^-}{C-1}. 
\end{align}
Consequently,
\begin{align}
\Delta H^-_i &= H(Y_i) - H(Y_i | Y_j = c_j, A_{ij} = 0)\\
&= \log C + h^- \log h^- + (1-h^-) \log \frac{1-h^-}{C-1}.
\end{align}
\end{proof}

\infotheorem*

\begin{proof}
From \Cref{lemma1} and \Cref{lemma2}, we aim to express $h^+$ and $h^-$ in terms of $s$, $d$, $r$, and $C$ by first deriving expressions for $p_{in}$ and $p_{out}$:

From $d = \frac{p_{in} + (C-1)p_{out}}{C}$, we have:
\begin{align}
p_{in} + (C-1)p_{out} = Cd
\end{align}
From $s = \frac{p_{in}}{p_{in} + (C-1)p_{out}}$, we have:
\begin{align}
p_{in} &= s(p_{in} + (C-1)p_{out})\\
&= sCd
\end{align}
And consequently:
\begin{align}
(C-1)p_{out} &= Cd - p_{in}\\
&= Cd - sCd\\
&= (1-s)Cd
\end{align}
Therefore:
\begin{align}
p_{out} &= \frac{(1-s)Cd}{C-1}
\end{align}
For positive edges, note that $h^+ = s$ directly from the definition. 

For negative edges, substituting these expressions into the formula for $h^-$:
\begin{align}
h^- &= \frac{1-p_{in}}{(1-p_{in}) + (C-1)(1-p_{out})}\\
&= \frac{1-sCd}{(1-sCd) + (C-1)(1-\frac{(1-s)Cd}{C-1})}\\
&= \frac{1-sCd}{(1-sCd) + (C-1) - (1-s)Cd}\\
&= \frac{1-sCd}{1-sCd + C-1 - Cd + sCd}\\
&= \frac{1-sCd}{C(1-d)}
\end{align}

In a graph with $n$ nodes where a fraction $r$ are labeled, the number of labeled nodes is $nr$ and the number of unlabeled nodes is $n(1 - r)$. As such, the total number of labeled-unlabeled node pairs is $nr n(1 - r) = n^2 r(1-r)$.

For any pair of nodes, the probability of having an edge between them is simply the edge density $d$, and the probability of having no edge is $(1-d)$. This can be verified:

\begin{align}
P(A_{ij} = 1) &= P(Y_i = c_j, Y_j = c_j)P(A_{ij} = 1| Y_i = c_j, Y_j = c_j) \\
&\quad + \sum_{c \neq c_j} P(Y_i = c, Y_j = c_j)P(A_{ij} = 1| Y_i = c, Y_j = c_j) \\
&= \frac{1}{C}(p_{in}) + \frac{C-1}{C}(p_{out})
\end{align}
Substituting our expressions for $p_{in}$ and $p_{out}$:
\begin{align}
P(A_{ij} = 1) &= \frac{1}{C}(sCd) + \frac{C-1}{C}(\frac{(1-s)Cd}{C-1})\\
&= \frac{sCd}{C} + \frac{(1-s)Cd}{C}\\
&= d
\end{align}
Similarly, $P(A_{ij} = 0) = 1-d$.

The expected total information gain from positive edges is the product of the number of labeled-unlabeled pairs, the probability of an edge, $\Delta H^+$, and the redundancy functions:
\begin{align}
IG_{pos\_total}(s,r,d) &= n^2r(1-r) d \Delta H^+(s,d) f_{pos}(s,r,d)
\end{align}
And the expected total information gain from negative edges is:
\begin{align}
IG_{neg\_total}(s,r,d) &= n^2r(1-r) (1-d) \Delta H^-(s,d) f_{neg}(s,r,d)
\end{align}
To get the expected gain per unlabeled node, we divide by the number of unlabeled nodes $n(1-r)$:
\begin{align}
IG_{pos}(s,r,d) &= \frac{n^2r(1-r) d \Delta H^+(s,d)f_{pos}(s,r,d)}{n(1-r)}\\
&= nr d \Delta H^+(s,d) f_{pos}(s,r,d)
\end{align}
And for negative edges:
\begin{align}
IG_{neg}(s,r,d) &= \frac{n^2r(1-r) (1-d) \Delta H^-(s,d)f_{neg}(s,r,d)}{n(1-r)}\\
&= nr (1-d) \Delta H^-(s,d) f_{neg}(s,r,d)
\end{align}
Where $\Delta H^+(s,d)$ and $\Delta H^-(s,d)$ are calculated by substituting our derived expressions for $h^+$ and $h^-$ into the formulas from \Cref{lemma1} and \Cref{lemma2}.
\end{proof}

\corr*
\begin{proof}
This follows directly from the expressions for $IG_{pos}(s,r,d)$ and $IG_{neg}(s,r,d)$ in \Cref{theorem1}. The total information gain per unlabeled node is $IG_{total}(s,r,d) = IG_{pos}(s,r,d) + IG_{neg}(s,r,d) = nr [d \Delta H^+(s,d)f_{pos}(s,r,d) + (1-d) \Delta H^-(s,d)f_{neg}(s,r,d)]$. Therefore, the relative contribution of negative edges is the ratio of $IG_{neg}(s,r,d)$ to $IG_{total}(s,r,d)$, which simplifies to the expression above.
\end{proof}

\section{Contrastive Message Passing Method Details}
\subsection{Learning through the Soft PSD Constraint}
\label{soft_psd_learning_algo}
\begin{algorithm}
\caption{Learning with Soft Positive Semi-Definite Constraints}
\label{soft_psd_algo}
\begin{algorithmic}
\REQUIRE Weight matrix $W$, node features $h_i$ and $h_j$, edge type $t \in \{\text{pos}, \text{neg}\}$, learnable scalar $\beta > 0$
\ENSURE Soft-PSD constrained matrix $\text{Soft-PSD}(W)$
\STATE Compute eigendecomposition: $W = Q\Lambda Q^{-1}$ \{$\Lambda$ is diagonal matrix of eigenvalues\}
\STATE Calculate cosine similarity: $c = \frac{h_i^T h_j}{||h_i|| \cdot ||h_j||}$
\STATE Compute downweighting factor $\tau$:
\begin{align}
\tau = \begin{cases}
\sigma\left(+c(h_i,h_j)(1 + \beta)\right) & \text{if } t = \text{pos} \\
\sigma\left(-c(h_i,h_j)(1 + \beta)\right) & \text{if } t = \text{neg}
\end{cases}
\end{align}
\STATE Modify eigenvalues:
\begin{align}
\hat{\Lambda}_{ii} = 
\begin{cases}
\Lambda_{ii} & \text{if } \Lambda_{ii} \geq 0 \\
\tau \cdot \Lambda_{ii} & \text{if } \Lambda_{ii} < 0
\end{cases}
\end{align}
\STATE Compute soft-PSD matrix: $\text{Soft-PSD}(W) = Q\hat{\Lambda}Q^{-1}$

\textbf{return} $\text{Soft-PSD}(W)$
\end{algorithmic}
\textbf{Usage in Model:} The returned $\text{Soft-PSD}(W)$ is then used in place of $W$ in the message passing operation. For example, 
\begin{align}
h_i^{(l+1)} &= \sum_{j \in \mathcal{N}_{\text{pos}}(i)} \text{Soft-PSD}(W_+)h_j^{(l)} - \sum_{k \in \mathcal{N}_{\text{neg}}(i)} \text{Soft-PSD}(W_-)h_k^{(l)}
\end{align}
\textbf{Learning:} Gradient descent updates $W$ and $\beta$ via backpropagation through the eigendecomposition
\end{algorithm}
Given a learnable weight matrix $W$, learning this matrix while still enforcing the soft PSD constraint is simple: (1) diagonalize $W$ into $Q$, $\Lambda$, and $Q^{-1}$, (2) downweight the negative eigenvalues in $\Lambda$ by $\tau$ where the resulting reweighted eigenvalues are given by $\hat{\Lambda}$, and (3) recombine the terms $Q$, $\hat{\Lambda}$, and $Q^{-1}$ to get your soft PSD matrix $\text{Soft-PSD}(W) = Q\hat{\Lambda}Q^{-1}$, which is then applied to the node features. Since all the operations are differentiable, we can backpropagate through this soft PSD constraint, making learning feasible. These operations are summarized in \Cref{soft_psd_algo}. 
\subsection{Design of $\tau$}
\label{tau_design}
Here, we aim to describe the intuition behind why $\tau$ is formulated as
\begin{align}
\tau = \begin{cases}
\sigma\left(+c(h_i,h_j)(1 + \beta)\right) & \text{for } W = W_{+} \\
\sigma\left(-c(h_i,h_j)(1 + \beta)\right) & \text{for } W = W_{-}
\end{cases}.
\end{align}

This formulation adaptively controls the PSD constraint based on node similarity. For positive edges, high node similarity produces higher $\tau$ values, relaxing the PSD constraint to prevent over-compression of already similar representations. For negative edges, high similarity results in lower $\tau$ values, enforcing stricter PSD constraints to ensure proper repulsion.

The design directly leverages cosine similarity $c(h_i,h_j) \in [-1,1]$ as input to the sigmoid function. As similarity increases toward $1$, positive edges receive less constraint ($\tau$ approaching $1$), while negative edges receive stricter constraint ($\tau$ approaching $0$). Conversely, as similarity decreases toward $-1$, positive edges receive stricter constraints while negative edges receive more flexibility.

The term $(1 + \beta)$ serves dual purposes: the constant $1$ ensures a minimum level of inductive bias even if $\beta$ approaches zero during training, preventing the model from losing its contrastive geometric properties. The learnable parameter $\beta > 0$ controls the sensitivity of the adaptation mechanism---higher $\beta$ values expand the input range to the sigmoid, creating more extreme $\tau$ values near $0$ or $1$ to enforce stronger constraints.
\subsection{Example CMP Implementations: GraphSAGE and GAT}
\label{cmp_implementations}
Throughout this section, weight matrices $W^{+}$ and $W^{-}$ are subject to the soft-PSD constraint defined in \Cref{soft_psd_section}.

\textbf{GraphSAGE-CMP.} The standard GraphSAGE update rule with mean aggregation is given by:
\begin{align}
    h_i^{(l + 1)} = W_{1} h_i^{(l)} + W_{2} \cdot \text{mean}_{j\in\mathcal{N}_{pos}(i)}h_j^{(l)}.
\end{align}
Our CMP implementation extends this to
\begin{align}
    h_i^{(l + 1)} = W h_i^{(l)} + W_{+}\cdot \text{mean}_{j\in\mathcal{N}_{pos}(i)}h_j^{(l)} - W_{-} \cdot \text{mean}_{k\in\mathcal{N}_{neg}(i)}h_k^{(l)},
\end{align}
enabling GraphSAGE to explicitly model both positive and negative relationships.

\textbf{GAT-CMP.} The standard GAT formulation computes node representations through attention-weighted aggregation:
\begin{align}
    h_i^{(l+1)} = \sum_{j \in \mathcal{N}(i) \cup \{i\}} \alpha_{ij}^{(l)} W^{(l)} h_j^{(l)}
\end{align}
where $\alpha_{ij}^{(l)}$ is the attention coefficient between nodes $i$ and $j$. CMP extends this to
\begin{align}
    h_i^{(l+1)} = \sum_{j \in \mathcal{N}_{pos}(i)\cup \{i\}} \alpha_{ij}^{+,(l)} W_{+}^{(l)} h_j^{(l)} - \sum_{k \in \mathcal{N}_{neg}(i)} \alpha_{ik}^{-,(l)} W_{-}^{(l)} h_k^{(l)}.
\end{align}
Here, $\alpha_{ij}^{+,(l)}$ and $\alpha_{ik}^{-,(l)}$ denote attention coefficients in layer $l$ computed separately for positive and negative edges, respectively.

\section{Experimental Details}
\label{app_model_details}
These specifications remain consistent across both the SBM experiments (\Cref{sbm_experiments}) and the real dataset experiments (\Cref{real_experiments}). 

\textbf{Model Architecture}. All models share the following components:

(1. Input Feature Transformation) An initial linear projection to map input features to hidden representation
\begin{equation}
    h_i^{(0)} = W_{\text{lift}} x_i + b_{\text{lift}}
\end{equation}
(2. Message Passing Layers) Two message passing layers whose form follows \Cref{cmp_implementations} with residual connections
\begin{equation}
    h_i^{(\ell+1)} = h_i^{(\ell)} + \text{MP}^{(\ell)}(h_i^{(\ell)}, \{h_j^{(\ell)} : j \in \mathcal{N}(i)\})
\end{equation}
(3. Normalization and Activation) Layer normalization followed by LeakyReLU activation
\begin{equation}
    \hat{h}_i^{(\ell)} = \text{LN}(h_i^{(\ell)})
\end{equation}
\begin{equation}
    \tilde{h}_i^{(\ell)} = \text{LeakyReLU}(\hat{h}_i^{(\ell)}, \alpha=0.2)
\end{equation}
(4. Output Projection) Final linear projection to output dimension
\begin{equation}
    z_i = W_{\text{proj}} h_i^{(L)} + b_{\text{proj}}
\end{equation}
\textbf{Hyperparameters.} Our experiments use the hyperparameters listed in \Cref{hyperparameters}.

\begin{table}[h]
\centering
\begin{tabular}{|l|c|}
\hline
\textbf{Hyperparameter} & \textbf{Value} \\
\hline
Hidden dimension & 64 \\
Learning rate & 0.01 \\
Weight decay & 5e-4 \\
Batch size & Full Batch \\
Optimizer & Adam \\
Activation function & LeakyReLU (negative slope=$0.2$) \\
Number of layers & 2 \\
Early stopping patience & 100 epochs \\
Maximum epochs & 200 \\
Random seeds & [42, 43, 44, 45, 46] \\
\hline
\end{tabular}
\caption{Hyperparameters used in all experiments}
\label{hyperparameters}
\end{table}
\textbf{Negative Edge Sampling.} For contrastive approaches, negative edges are sampled using the negative sampling technique from PyTorch Geometric \cite{fey2019fast}. Specifically, we randomly sample from the negative edge set:
\begin{equation}
    \mathcal{E}^- = \{(i, j) \mid i, j \in \mathcal{V}, (i, j) \notin \mathcal{E}^+\}
\end{equation}
We sample the same number of negative edges as positive edges.

\textbf{Computational Resources}. All experiments were conducted on one NVIDIA A100 machine. 

\textbf{Real Dataset Graph Characteristics}. We provide an outline of the graph characteristics of each dataset. Note that these can often be computed as a preprocessing step and provide guidance on whether one should use CMP on the dataset. For instance, if a dataset has favorable attributes as defined by \Cref{section3} (aka high homophily, edge density, and low label rates), we can expect that CMP will perform well; conversely, if a dataset has low homphily or edge density, a standard GNN may be favorable. 

\begin{table}[h]
\centering
\begin{tabular}{|l|r|r|r|r|l|}
\hline
\textbf{Dataset} & \textbf{Nodes} & \textbf{Edges} & \textbf{Features} & \textbf{Classes} \\
\hline
Cora & 2,708 & 10,556 & 1,433 & 7 \\
Amazon Photos & 7,650 & 238,162 & 745 & 8 \\
Amazon Computers & 13,752 & 491,722 & 767 & 10 \\
Coauthor Physics & 34,493 & 495,924 & 8,415 & 5 \\
Coauthor CS & 18,333 & 163,788 & 6,805 & 15 \\
PPI & $\sim$ 2,245.3 & $\sim$ 61,318.4 & 50 & 121 (\text{\#tasks}) \\
CiteSeer & 3,327 & 9,104 & 3,703 & 6 \\
PubMed & 19,717 & 88,648 & 500 & 3 \\
Flickr & 89,250 & 899,756 & 500 & 7 \\
\hline
\end{tabular}
\caption{Graph Dataset Statistics}
\label{dataset_description_table}
\end{table}
\begin{table}[h]
\centering
\caption{Comparison of Graph Dataset Characteristics}
\begin{tabular}{lrrrrrr}
\hline
\textbf{Dataset} & \textbf{Modularity} & \textbf{Clust.} & \textbf{Density} & \textbf{Avg.} & \textbf{Homophily} & \textbf{Intra/Inter} \\
 & & \textbf{Coef.} & \textbf{($10^{-3}$)} & \textbf{Degree} & \textbf{Ratio} & \textbf{Ratio} \\
\hline
Cora & 0.64 & 0.24 & 1.44 & 3.9 & 19.5 & 4.3 \\
Amazon Photos & 0.63 & 0.40 & 4.07 & 31.1 & 24.3 & 4.8 \\
Amazon Computers & 0.48 & 0.34 & 2.60 & 35.8 & 13.3 & 3.5 \\
Coauthor Physics & 0.47 & 0.38 & 0.42 & 14.4 & 28.7 & 13.6 \\
Coauthor CS & 0.70 & 0.34 & 0.49 & 8.9 & 33.3 & 4.2 \\
PPI & 0.01 & 0.16 & 10.36 & 18.3 & 2.6 & 3.2 \\
CiteSeer & 0.54 & 0.14 & 0.82 & 2.7 & 12.9 & 2.8 \\
PubMed & 0.43 & 0.06 & 0.23 & 4.5 & 7.3 & 4.1 \\
Flickr & 0.07 & 0.03 & 0.11 & 10.1 & 1.3 & 0.5 \\
\hline
\end{tabular}
\label{real_dataset_graph_table}
\end{table} 
The graph dataset statistics were taken from \cite{fey2019fast} and the graph dataset characteristics were calculated directly using \cite{hagberg2008exploring}.

\section{Additional Experiments}
\subsection{GAT Results for Real Datasets}
\label{gat_results_real_datasets} 
\begin{table}[t]
\centering
\caption{Median test accuracy performance of GAT across real datasets and label rates (\% labels used from total labels). Bold values indicate the best performing model.}
\label{real_data_acc_gat_results}
\resizebox{\textwidth}{!}{
\definecolor{lightgray}{gray}{0.9}
\definecolor{blue}{rgb}{0.7, 0.8, 1.0}
\definecolor{orange}{rgb}{1.0, 0.8, 0.6}
\definecolor{purple}{rgb}{0.9, 0.7, 0.9}
\begin{tabular}{l|l|c|c|c|c|c|c|c|c|c}
\toprule
Dataset & Model & 0.1\% & 0.2\% & 0.5\% & 1\% & 2\% & 5\% & 10\% & 20\% & 50\% \\
\midrule
Cora & CMP & \cellcolor{lightgray}{--} & \cellcolor{lightgray}{--} & \cellcolor{lightgray}{--} & \cellcolor{purple}{56.45} & \cellcolor{purple}{\textbf{67.32}} & \cellcolor{purple}{\textbf{77.72}} & \cellcolor{purple}{\textbf{80.44}} & \cellcolor{blue}{\textbf{82.29}} & \cellcolor{blue}{\textbf{85.98}} \\
& Standard & \cellcolor{lightgray}{--} & \cellcolor{lightgray}{--} & \cellcolor{lightgray}{--} & \cellcolor{purple}{52.59} & \cellcolor{purple}{61.16} & \cellcolor{purple}{71.21} & \cellcolor{purple}{77.86} & \cellcolor{blue}{81.71} & \cellcolor{blue}{84.78} \\
& Unconstrained & \cellcolor{lightgray}{--} & \cellcolor{lightgray}{--} & \cellcolor{lightgray}{--} & \cellcolor{purple}{48.11} & \cellcolor{purple}{65.86} & \cellcolor{purple}{72.82} & \cellcolor{purple}{77.77} & \cellcolor{blue}{81.08} & \cellcolor{blue}{83.58} \\
& CL & \cellcolor{lightgray}{--} & \cellcolor{lightgray}{--} & \cellcolor{lightgray}{--} & \cellcolor{purple}{\textbf{63.58}} & \cellcolor{purple}{65.48} & \cellcolor{purple}{68.74} & \cellcolor{purple}{71.77} & \cellcolor{blue}{77.07} & \cellcolor{blue}{82.47} \\
\midrule
CiteSeer & CMP & \cellcolor{lightgray}{--} & \cellcolor{lightgray}{--} & \cellcolor{lightgray}{--} & \cellcolor{orange}{52.16} & \cellcolor{orange}{58.11} & \cellcolor{orange}{63.84} & \cellcolor{orange}{66.73} & \cellcolor{orange}{68.84} & \cellcolor{orange}{72.00} \\
& Standard & \cellcolor{lightgray}{--} & \cellcolor{lightgray}{--} & \cellcolor{lightgray}{--} & \cellcolor{orange}{48.45} & \cellcolor{orange}{57.70} & \cellcolor{orange}{\textbf{66.63}} & \cellcolor{orange}{\textbf{67.78}} & \cellcolor{orange}{\textbf{70.77}} & \cellcolor{orange}{\textbf{73.20}} \\
& Unconstrained & \cellcolor{lightgray}{--} & \cellcolor{lightgray}{--} & \cellcolor{lightgray}{--} & \cellcolor{orange}{48.21} & \cellcolor{orange}{52.37} & \cellcolor{orange}{62.32} & \cellcolor{orange}{64.48} & \cellcolor{orange}{67.47} & \cellcolor{orange}{70.35} \\
& CL & \cellcolor{lightgray}{--} & \cellcolor{lightgray}{--} & \cellcolor{lightgray}{--} & \cellcolor{orange}{\textbf{58.64}} & \cellcolor{orange}{\textbf{59.88}} & \cellcolor{orange}{61.33} & \cellcolor{orange}{66.09} & \cellcolor{orange}{67.00} & \cellcolor{orange}{70.65} \\
\midrule
PubMed & CMP & \cellcolor{orange}{62.29} & \cellcolor{orange}{70.36} & \cellcolor{orange}{74.27} & \cellcolor{orange}{75.88} & \cellcolor{orange}{76.98} & \cellcolor{orange}{79.27} & \cellcolor{lightgray}{--} & \cellcolor{lightgray}{--} & \cellcolor{lightgray}{--} \\
& Standard & \cellcolor{orange}{58.64} & \cellcolor{orange}{68.20} & \cellcolor{orange}{\textbf{75.90}} & \cellcolor{orange}{\textbf{77.61}} & \cellcolor{orange}{\textbf{79.50}} & \cellcolor{orange}{\textbf{82.50}} & \cellcolor{lightgray}{--} & \cellcolor{lightgray}{--} & \cellcolor{lightgray}{--} \\
& Unconstrained & \cellcolor{orange}{59.81} & \cellcolor{orange}{67.10} & \cellcolor{orange}{71.62} & \cellcolor{orange}{72.82} & \cellcolor{orange}{75.51} & \cellcolor{orange}{78.42} & \cellcolor{lightgray}{--} & \cellcolor{lightgray}{--} & \cellcolor{lightgray}{--} \\
& CL & \cellcolor{orange}{\textbf{69.43}} & \cellcolor{orange}{\textbf{71.14}} & \cellcolor{orange}{75.41} & \cellcolor{orange}{77.41} & \cellcolor{orange}{79.57} & \cellcolor{orange}{80.92} & \cellcolor{lightgray}{--} & \cellcolor{lightgray}{--} & \cellcolor{lightgray}{--} \\
\midrule
CS & CMP & \cellcolor{purple}{\textbf{67.12}} & \cellcolor{purple}{\textbf{74.69}} & \cellcolor{purple}{\textbf{84.04}} & \cellcolor{purple}{\textbf{86.01}} & \cellcolor{blue}{\textbf{89.10}} & \cellcolor{blue}{\textbf{91.18}} & \cellcolor{lightgray}{--} & \cellcolor{lightgray}{--} & \cellcolor{lightgray}{--} \\
& Standard & \cellcolor{purple}{41.52} & \cellcolor{purple}{59.26} & \cellcolor{purple}{74.37} & \cellcolor{purple}{81.07} & \cellcolor{blue}{85.03} & \cellcolor{blue}{88.40} & \cellcolor{lightgray}{--} & \cellcolor{lightgray}{--} & \cellcolor{lightgray}{--} \\
& Unconstrained & \cellcolor{purple}{60.22} & \cellcolor{purple}{73.13} & \cellcolor{purple}{81.60} & \cellcolor{purple}{84.86} & \cellcolor{blue}{86.75} & \cellcolor{blue}{89.81} & \cellcolor{lightgray}{--} & \cellcolor{lightgray}{--} & \cellcolor{lightgray}{--} \\
& CL & \cellcolor{purple}{54.42} & \cellcolor{purple}{56.22} & \cellcolor{purple}{53.05} & \cellcolor{purple}{54.61} & \cellcolor{blue}{59.55} & \cellcolor{blue}{48.80} & \cellcolor{lightgray}{--} & \cellcolor{lightgray}{--} & \cellcolor{lightgray}{--} \\
\midrule
Physics & CMP & \cellcolor{purple}{\textbf{87.40}} & \cellcolor{purple}{\textbf{90.83}} & \cellcolor{purple}{\textbf{92.76}} & \cellcolor{purple}{\textbf{94.07}} & \cellcolor{blue}{\textbf{94.68}} & \cellcolor{blue}{\textbf{95.14}} & \cellcolor{lightgray}{--} & \cellcolor{lightgray}{--} & \cellcolor{lightgray}{--} \\
& Standard & \cellcolor{purple}{80.47} & \cellcolor{purple}{80.26} & \cellcolor{purple}{87.71} & \cellcolor{purple}{91.26} & \cellcolor{blue}{92.50} & \cellcolor{blue}{93.83} & \cellcolor{lightgray}{--} & \cellcolor{lightgray}{--} & \cellcolor{lightgray}{--} \\
& Unconstrained & \cellcolor{purple}{85.20} & \cellcolor{purple}{87.79} & \cellcolor{purple}{92.07} & \cellcolor{purple}{93.46} & \cellcolor{blue}{94.19} & \cellcolor{blue}{94.79} & \cellcolor{lightgray}{--} & \cellcolor{lightgray}{--} & \cellcolor{lightgray}{--} \\
& CL & \cellcolor{purple}{83.61} & \cellcolor{purple}{83.24} & \cellcolor{purple}{86.32} & \cellcolor{purple}{92.07} & \cellcolor{blue}{93.24} & \cellcolor{blue}{94.02} & \cellcolor{lightgray}{--} & \cellcolor{lightgray}{--} & \cellcolor{lightgray}{--} \\
\midrule
Computers & CMP & \cellcolor{purple}{\textbf{50.47}} & \cellcolor{purple}{\textbf{64.72}} & \cellcolor{purple}{\textbf{77.68}} & \cellcolor{purple}{\textbf{82.98}} & \cellcolor{blue}{\textbf{85.96}} & \cellcolor{blue}{\textbf{87.41}} & \cellcolor{lightgray}{--} & \cellcolor{lightgray}{--} & \cellcolor{lightgray}{--} \\
& Standard & \cellcolor{purple}{36.62} & \cellcolor{purple}{39.91} & \cellcolor{purple}{50.66} & \cellcolor{purple}{63.15} & \cellcolor{blue}{70.20} & \cellcolor{blue}{76.95} & \cellcolor{lightgray}{--} & \cellcolor{lightgray}{--} & \cellcolor{lightgray}{--} \\
& Unconstrained & \cellcolor{purple}{39.40} & \cellcolor{purple}{58.43} & \cellcolor{purple}{69.10} & \cellcolor{purple}{79.46} & \cellcolor{blue}{82.56} & \cellcolor{blue}{84.93} & \cellcolor{lightgray}{--} & \cellcolor{lightgray}{--} & \cellcolor{lightgray}{--} \\
& CL & \cellcolor{purple}{46.47} & \cellcolor{purple}{38.72} & \cellcolor{purple}{37.06} & \cellcolor{purple}{36.50} & \cellcolor{blue}{36.52} & \cellcolor{blue}{37.42} & \cellcolor{lightgray}{--} & \cellcolor{lightgray}{--} & \cellcolor{lightgray}{--} \\
\midrule
Photo & CMP & \cellcolor{purple}{38.94} & \cellcolor{purple}{\textbf{55.17}} & \cellcolor{purple}{\textbf{82.01}} & \cellcolor{purple}{\textbf{86.97}} & \cellcolor{blue}{\textbf{90.05}} & \cellcolor{blue}{\textbf{91.36}} & \cellcolor{lightgray}{--} & \cellcolor{lightgray}{--} & \cellcolor{lightgray}{--} \\
& Standard & \cellcolor{purple}{29.01} & \cellcolor{purple}{30.00} & \cellcolor{purple}{48.72} & \cellcolor{purple}{60.70} & \cellcolor{blue}{75.04} & \cellcolor{blue}{84.85} & \cellcolor{lightgray}{--} & \cellcolor{lightgray}{--} & \cellcolor{lightgray}{--} \\
& Unconstrained & \cellcolor{purple}{29.62} & \cellcolor{purple}{44.64} & \cellcolor{purple}{72.76} & \cellcolor{purple}{81.97} & \cellcolor{blue}{87.15} & \cellcolor{blue}{90.02} & \cellcolor{lightgray}{--} & \cellcolor{lightgray}{--} & \cellcolor{lightgray}{--} \\
& CL & \cellcolor{purple}{\textbf{48.60}} & \cellcolor{purple}{54.43} & \cellcolor{purple}{46.60} & \cellcolor{purple}{43.25} & \cellcolor{blue}{27.20} & \cellcolor{blue}{35.57} & \cellcolor{lightgray}{--} & \cellcolor{lightgray}{--} & \cellcolor{lightgray}{--} \\
\midrule
PPI & CMP & \cellcolor{purple}{\textbf{45.99}} & \cellcolor{purple}{46.51} & \cellcolor{purple}{48.17} & \cellcolor{purple}{\textbf{50.81}} & \cellcolor{blue}{\textbf{53.20}} & \cellcolor{blue}{\textbf{55.01}} & \cellcolor{lightgray}{--} & \cellcolor{lightgray}{--} & \cellcolor{lightgray}{--} \\
& Standard & \cellcolor{purple}{44.34} & \cellcolor{purple}{45.36} & \cellcolor{purple}{46.16} & \cellcolor{purple}{47.92} & \cellcolor{blue}{51.41} & \cellcolor{blue}{53.08} & \cellcolor{lightgray}{--} & \cellcolor{lightgray}{--} & \cellcolor{lightgray}{--} \\
& Unconstrained & \cellcolor{purple}{45.50} & \cellcolor{purple}{\textbf{46.88}} & \cellcolor{purple}{\textbf{48.34}} & \cellcolor{purple}{49.80} & \cellcolor{blue}{51.54} & \cellcolor{blue}{54.53} & \cellcolor{lightgray}{--} & \cellcolor{lightgray}{--} & \cellcolor{lightgray}{--} \\
& CL & \cellcolor{purple}{41.81} & \cellcolor{purple}{41.26} & \cellcolor{purple}{43.31} & \cellcolor{purple}{44.65} & \cellcolor{blue}{45.92} & \cellcolor{blue}{46.87} & \cellcolor{lightgray}{--} & \cellcolor{lightgray}{--} & \cellcolor{lightgray}{--} \\
\midrule
Flickr & CMP & \cellcolor{orange}{42.31} & \cellcolor{orange}{\textbf{43.75}} & \cellcolor{orange}{43.38} & \cellcolor{orange}{\textbf{44.56}} & \cellcolor{orange}{44.96} & \cellcolor{orange}{\textbf{45.50}} & \cellcolor{lightgray}{--} & \cellcolor{lightgray}{--} & \cellcolor{lightgray}{--} \\
& Standard & \cellcolor{orange}{42.26} & \cellcolor{orange}{42.84} & \cellcolor{orange}{42.95} & \cellcolor{orange}{44.06} & \cellcolor{orange}{44.73} & \cellcolor{orange}{45.44} & \cellcolor{lightgray}{--} & \cellcolor{lightgray}{--} & \cellcolor{lightgray}{--} \\
& Unconstrained & \cellcolor{orange}{42.25} & \cellcolor{orange}{42.72} & \cellcolor{orange}{43.21} & \cellcolor{orange}{43.58} & \cellcolor{orange}{44.66} & \cellcolor{orange}{45.40} & \cellcolor{lightgray}{--} & \cellcolor{lightgray}{--} & \cellcolor{lightgray}{--} \\
& CL & \cellcolor{orange}{\textbf{42.49}} & \cellcolor{orange}{42.65} & \cellcolor{orange}{\textbf{43.48}} & \cellcolor{orange}{44.41} & \cellcolor{orange}{\textbf{45.00}} & \cellcolor{orange}{45.08} & \cellcolor{lightgray}{--} & \cellcolor{lightgray}{--} & \cellcolor{lightgray}{--} \\
\bottomrule
\end{tabular}
}
\vskip -0.15in
\end{table}

In \Cref{real_experiments}, we provided results for GraphSAGE. Here, we additionally report the results for GAT evaluated on the real datasets; results are recorded in \Cref{real_data_acc_gat_results}. The setup is identical to \Cref{real_experiments}, except with the GAT architecture rather than GraphSage.
\subsection{Stochastic Block Model: Heterophily}
\label{heterophily_sbm_exp}
\begin{figure}
  \centering
\includegraphics[width=0.95\textwidth]{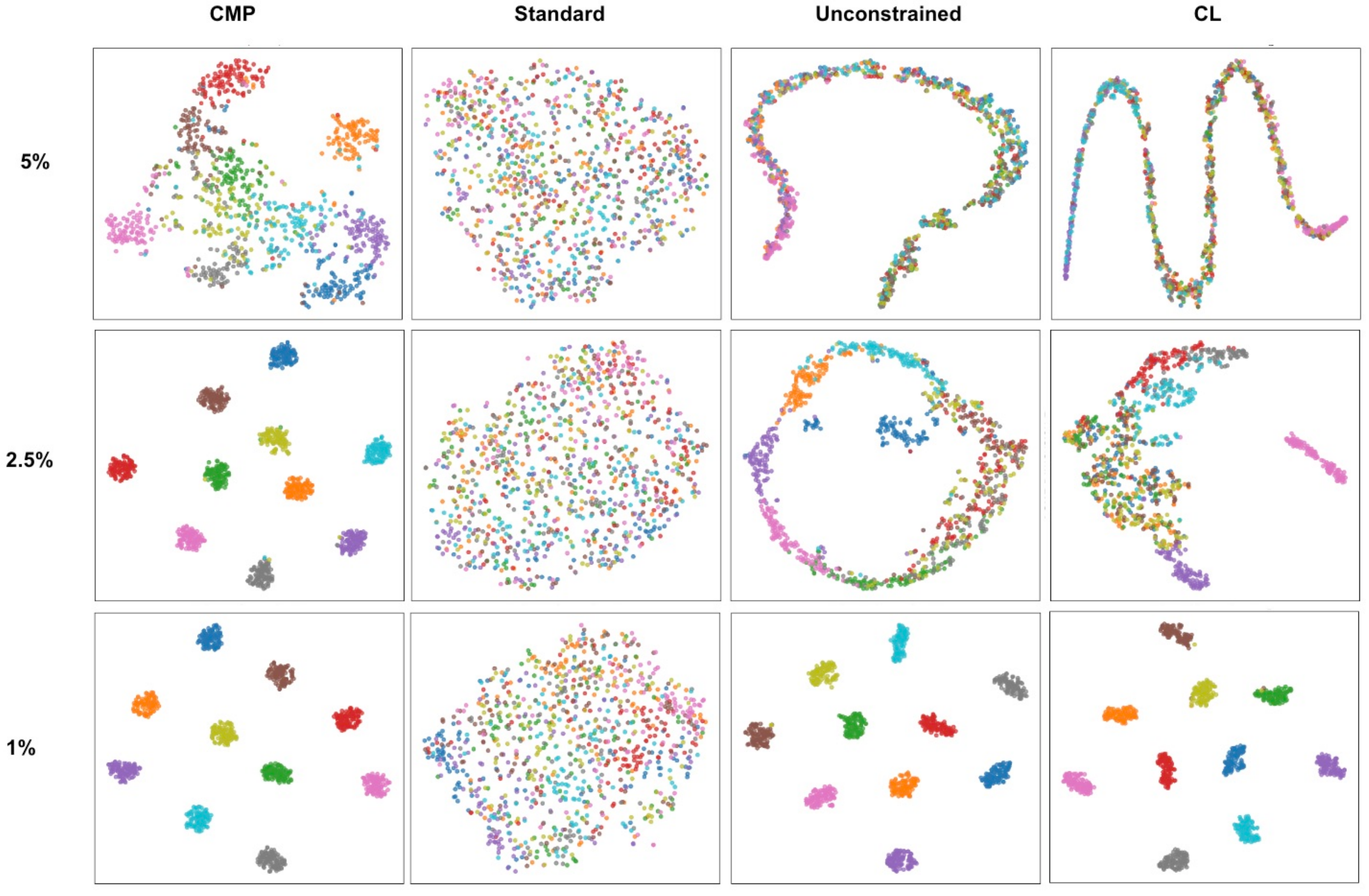}
  \caption{First two t-SNE dimensions of the Stochastic Block Model's learned embeddings taken from the last message passing layer of GraphSage. In each row are the node embeddings at a specific heterophily ratio.} 
  \label{sbm_embeds_heterophily}
  \vskip -0.2in
\end{figure}
\begin{figure}
  \centering
\includegraphics[width=\textwidth]{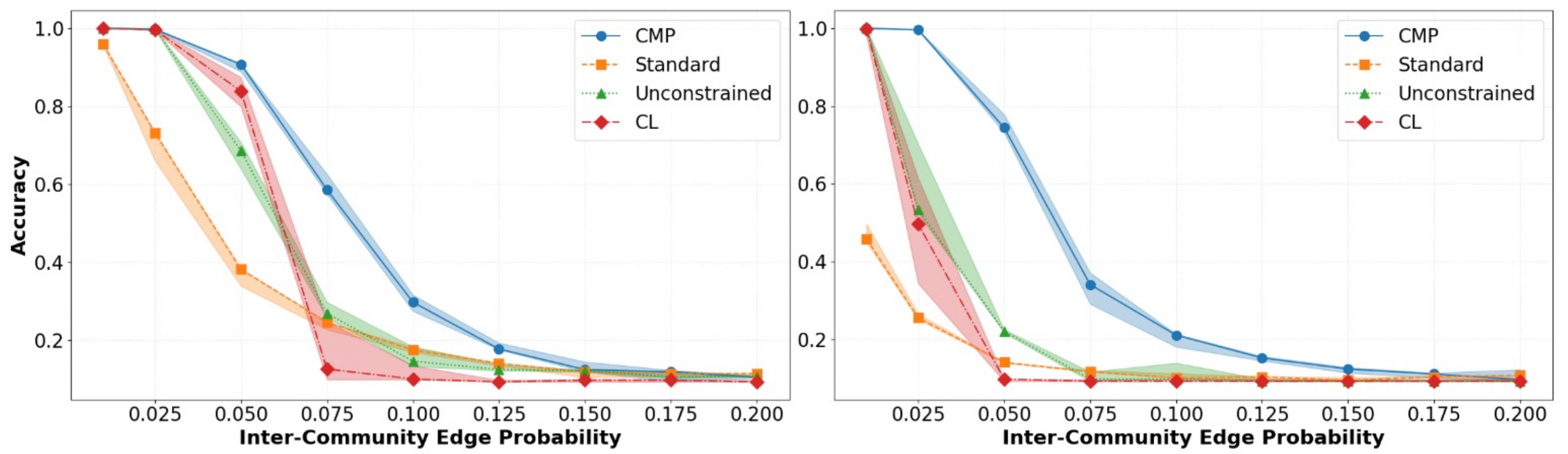}
  \caption{Class prediction accuracy of CMP and baselines on the Stochastic Block Model over different inter-community edge probabilities on the (left) GAT-CMP and (right) GraphSage-CMP architectures. The line and shaded regions show the median and $25$/$75$th percentile values, respectively.} 
  \label{sbm_heterophily_perf}
\end{figure}
Here, we investigate how graph heterophily affects the performance of CMP and baseline methods. This experiment was designed to systematically investigate the relationship between heterophily levels and the information gain provided by negative edges, as predicted by our theoretical analysis in \Cref{section3}.

\textbf{Setup}. We generated synthetic data using Stochastic Block Models (SBMs) with 1000 nodes distributed equally across 10 communities (100 nodes per community). Each node was assigned a 32-dimensional feature vector randomly initialized from $\mathcal{N}(0,1)$ distribution. For training, we used a fixed label rate of 20\% of the nodes. The intra-community edge probability was fixed at $p_{in} = 0.25$, while we varied the inter-community edge probability $p_{out}$ across the values \{0.01, 0.025, 0.05, 0.075, 0.1, 0.125, 0.15, 0.175, 0.2\} to create graphs with different heterophily levels. By fixing $p_{in}$ and varying $p_{out}$, we created graphs with different heterophily levels. Lower $p_{out}$ values correspond to higher homophily, where nodes are more likely to connect with others in the same community.

\textbf{Results}. As expected from \Cref{section3} findings, the accuracy of baselines shown in \Cref{sbm_heterophily_perf} decreases as the heterophily (i.e. inter-community edge probability) increases. Interestingly, we see that CMP continues to perform better than baselines even as we increase heterophily until very high heterophily levels where all methods converge to $0$. This performance could be explained by visualizing the learned node embeddings (shown in \Cref{sbm_embeds_heterophily}) where we see that at even high levels of heterophily ($5\%$), the community clusters are fairly well defined in CMP compared to other methods. 

\section{Additional Figures}
We provide additional geometric intuition for the soft PSD constraint in \Cref{soft_psd_fig}.
\begin{figure}
  \centering
\includegraphics[width=0.9\textwidth]{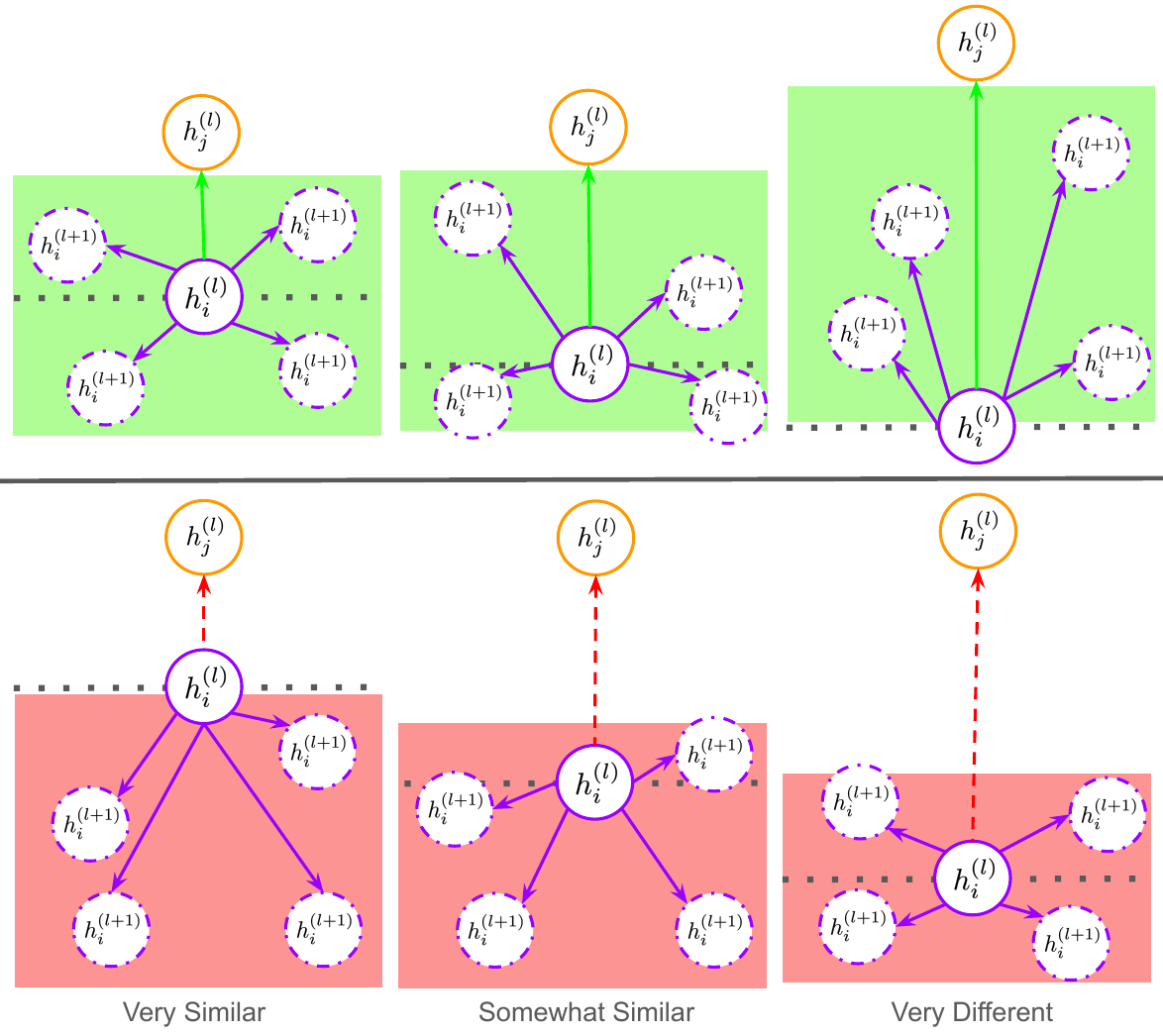}
  \caption{Soft positive semidefinite constraints modulate the allowed transformations based on node similarity. For positive edges (top row): when nodes are already similar (left), the constraint relaxes ($\tau^+ \approx 1$), allowing any linear transformation; when nodes are dissimilar (right), strict PSD is enforced ($\tau^+ \approx 0$) to pull the node embeddings closer together. For negative edges (bottom row): when nodes are similar (left), strict PSD is enforced ($\tau^- \approx 0$) to push node embeddings away from each other; when nodes are already well-separated (right), the constraint relaxes ($\tau^- \approx 1$) to enable arbitrary linear transformations and avoid unnecessary repulsion. Green and red regions indicate the permissible transformation halfspaces under the respective constraints.} 
  \label{soft_psd_fig}
\end{figure}
\section{Limitations}
\label{app_limitations}
\textbf{Random Negative Edges}. The current implementation of CMP rely on random sampling of negative edges (i.e., non-existing edges) between nodes. This approach does not account for the semantic meaning of these non-connections, which might not always represent true negative relationships. In real-world graphs and even in SBMs (as shown in \Cref{section3}), the absence of an edge can result from incomplete observations rather than a meaningful negative relationship. Future work could explore informed negative sampling strategies that consider node feature similarity, community structure, or domain knowledge to identify more informative negative edges. 

\textbf{Computational Complexity}. The computational complexity of CMP is linear with respect to the number of edges (depending on the negative edge sampling strategy; in all our experiments in \Cref{experiments}, we sample an equivalent number of negative edges as there are positive edges) and cubic with respect to the feature dimension due to the eigendecomposition required for the soft-PSD constraint. In practice, since most GNNs operate with moderate feature dimensions (typically $32$-$512$), this cubic component does not significantly impact efficiency compared to standard message passing operations. However, for high-dimensional node features, this could become a bottleneck. One potential optimization would be to use approximate eigendecomposition methods or alternative formulations of the soft-PSD constraint that avoid the cubic complexity.

\textbf{Graph Structure Dependence}. As demonstrated in \Cref{section3} and \Cref{experiments}, CMP's effectiveness is highly dependent on specific graph properties like homophily ratio, edge density, and label rates. In graphs with low homophily or extremely sparse connectivity, CMP may not outperform standard GNNs that only use positive edges. This dependence on graph structure limits the one-size-fits-all applicability of CMP, and practitioners should evaluate graph characteristics (as a preprocessing step as was done in \Cref{real_dataset_graph_table}) before adopting the method.

\end{document}